\documentclass[final,3p,times]{elsarticle}

\usepackage{amssymb}
\usepackage{amsthm}
\usepackage{amsmath}
\usepackage{mathrsfs}
\usepackage{stmaryrd}
\usepackage{graphicx,amsfonts,listings,xcolor,colortbl}

\usepackage{multirow}
\usepackage{booktabs}
\usepackage{blkarray}
\usepackage{arydshln}
\usepackage{subfigure}
\usepackage{fancyhdr}
\usepackage{setspace}
\usepackage{multirow}
\usepackage{color}
\usepackage{bbold}
\usepackage{caption}
\usepackage{epstopdf}

\usepackage{threeparttable}

\usepackage{multicol}
\usepackage{times}

\usepackage{soul}
\usepackage{url}
\usepackage[utf8]{inputenc}
\usepackage{makecell}
\usepackage{float}
\usepackage{stfloats}
\usepackage[OT1]{fontenc}
\usepackage{bm}
\usepackage{algorithm}
\usepackage{algpseudocode}
\usepackage[justification=centering]{caption}

\newtheorem {definition} {\bf Definition}[section]

\newtheorem {lemma} {\bf Lemma}
\newtheorem {theorem} {\bf Theorem}[section]

\begin{document}

\begin{frontmatter}



\title{Incremental Unsupervised Feature Selection for Dynamic Incomplete Multi-view Data}


\author[SWUFE]{Yanyong Huang   }
\ead{huangyy@swufe.edu.cn}
\author[SWUFE]{Kejun~Guo}
\ead{guokejun001@163.com}
\author[JD]{Xiuwen Yi }
\ead{xiuwenyi@foxmail.com}
\author[MinNan,Hagen]{Zhong Li}
\ead{zhong.li@fernuni-hagen.de}
\author[SWJTU1]{Tianrui Li \corref{cor1}}
\ead{trli@swjtu.edu.cn}

\cortext[cor1]{Corresponding author.}
\address[SWUFE]{School of Statistics, Southwestern University of Finance and Economics, Chengdu 611130, China}
\address[JD]{JD Intelligent Cities Research and JD Intelligent Cities Business Unit, Beijing, 100176, China}
\address[MinNan]{School of Mathematics and Statistics, Minnan Normal University, Zhangzhou 363000, China}
\address[Hagen]{Faculty of Mathematics and Computer Science, FernUniversit\"{a}t in Hagen, Hagen 58097, Germany}
\address[SWJTU1]{School of Computing and Artificial Intelligence, Southwest Jiaotong University, Chengdu 611756, China}

\begin{abstract}
Multi-view unsupervised feature selection has been proven to be efficient in reducing the dimensionality of multi-view unlabeled data with high dimensions. The previous methods assume all of the views are complete. However, in real applications, the multi-view data are often incomplete, \emph{i.e.}, some views of instances are missing, which will result in the failure of these methods. Besides, while the data arrive in form of streams, these existing methods will suffer the issues of high storage cost and expensive computation time. To address these issues, we propose an Incremental Incomplete Multi-view Unsupervised Feature Selection method (I$^2$MUFS) on incomplete multi-view streaming data. By jointly considering the consistent and complementary information across different views, I$^2$MUFS embeds the unsupervised feature selection into an extended weighted non-negative matrix factorization model, which can learn a consensus clustering indicator matrix and fuse different latent feature matrices with adaptive view weights. Furthermore, we introduce the incremental leaning mechanisms to develop an alternative iterative algorithm, where the feature selection matrix is incrementally updated, rather than recomputing on the entire updated data from scratch. A series of experiments are conducted to verify the effectiveness of the proposed method by comparing with several state-of-the-art methods. The experimental results demonstrate the effectiveness and efficiency of the proposed method in terms of the clustering metrics and the computational cost.
\end{abstract}


\end{frontmatter}


\section{Introduction}
With the increase of the ability to obtain data, a great amount of data can be collected from multiple sources or described by different feature sets. For instance, an image is characterized by different heterogenous features including visual features and text representation~\cite{rubino20173d}; the same news are reported by different languages in numerous countries~\cite{kim2010multi}. This type of data is usually called multi-view data. In many real applications, obtaining a large number of labeled data is difficulty or even infeasible~\cite{cai2010unsupervised, Tang:Unsupervised}. Multi-view unlabeled data are usually equipped with high-dimensional feature space, which will result in the problem of ``curse of dimensionality''~\cite{bellman:dynamic} and deteriorate the performance for the subsequent data mining or machine learning tasks, like clustering and information retrieval~\cite{zhang2019unsupervised, FAN2021MultiLabel}. Hence, how to efficiently reduce the dimensionality of multi-view unlabeled data has become an important and practical problem.

Multi-view unsupervised feature selection (MUFS) solves this problem by choosing a compact set of representative features from different heterogenous features. Many related methods have been developed in recent years. These methods mainly fall into two categories. The first is that concatenating multiple feature sets deriving from different views into a single-view and making use of the typical single-view unsupervised feature selection on the combined data~\cite{li2012NDFS,UDFS}. The second is that dealing with multi-view data for feature selection directly~\cite{Rui2019Feature, Sun2018Multi, wang2013multi}. These approaches usually learned the local structure of data to perform feature selection.

The aforementioned methods assume all of views are complete, namely, each data sample exists in all views. However, in practical applications, not all samples appear in each view~\cite{sampleAppears,Wen2019UnifiedEA}. For example, in a three-view image dataset with color texture, shape and textual features, some images are only described by a part of three features. Since the incompleteness of views, these existing MUFS methods could not be applied directly. It brings the challenge how to select features by the utilization of consistent and complementary information under the incomplete scenario. Furthermore, in a dynamic data environment, the data arrive in a streaming fashion. These existing MUFS methods based on the full batch manner need load all data at once, which will lead to the high storage cost and expensive computation time. In order to deal with these issues, we propose an Incremental Incomplete Multi-view Unsupervised Feature Selection method (I$^2$MUFS) on dynamic incomplete multi-view  data. Specifically, the proposed I$^2$MUFS method integrates unsupervised feature selection into an extended weighted non-negative matrix factorization (WNMF) based clustering objective function with graph regularization. By employing the consistent information across different views, I$^2$MUFS can learn a consensus clustering indicator matrix. And it fuses the latent feature matrices with adaptive view weights based on the complementarity among views. Furthermore, while the incomplete multi-view data arrive chunk-by-chunk, an incremental algorithm of I$^2$MUFS is developed for the update of feature selection matrices by utilizing  the previous computational results instead of recomputing in the whole data set from scratch.

The main contribution of this paper can be summarized as follows. (\romannumeral1) By jointly considering the consistent and complementary information across different views, we propose a novel multi-view unsupervised feature selection method for incomplete multi-view streaming data, which integrates the learning of consensus indicator matrix and the fusion of different latent feature matrices with adaptive view weights into an extended non-negative matrix factorization model. The learned consistent and complementary information can enhance the performance of feature selection. (\romannumeral2) By introducing the incremental learning mechanism, we develop an iterative optimization algorithm, which can greatly improve the computational efficiency of feature selection. We also give the theoretical proof to guarantee its convergence. (\romannumeral3) Experimental results show the effectiveness of the proposed method by comparing with the state-of-the-art (SOTA) methods in terms of the clustering metrics and the computational cost.

The rest of this paper is organized as follows. In Section 2, some related multi-view unsupervised feature selection methods are briefly reviewed. The detail of the proposed method is presented in Section 3. Section 4 introduces the alternative iterative optimization algorithm to solve the proposed objective function. Section 5 gives the theoretical analysis of the proposed method including the convergence and tim complexity analysis. Extensive experiments are conducted to demonstrate the effectiveness and efficiency of our method in Section 6. Section 7 concludes the paper.

\section{Related Works}
In this section, we briefly review some representative unsupervised feature selection methods. Feature selection, also known as feature (variable) reduction, aims to select a subset of  representative and important features from the whole feature space~\cite{FSBook2015,FSBook2008}. According to whether or not labels are provided in the procedure of selecting features, feature selection can be classified into three categories, namely, supervised methods~\cite{FS2020}, semi-supervised methods~\cite{FS2017} and unsupervised methods~\cite{S2019A}. In this paper, we focus on unsupervised feature selection. There are two kinds of unsupervised feature selection methods for multi-view data. The first one uses the unsupervised single-view feature selection method on the combined data, where multi-view features are concatenated to a single-view feature. The single-view unsupervised feature selection method usually is classified into three categories: filter-based methods, wrapper-based methods and embedded-based methods. The filter-based feature selection methods use different evaluation criterions to measure the importance of selected features. LapScore \cite{He:Laplacian} proposes a metric called Laplacian Score to measure the capability of local similarity preserving with regards to different features. Gu et al. employed the generalized Fisher score to select features~\cite{gu2012generalized}. The filter-based method did not consider the follow-up learning algorithms, which could not achieve the satisfied performance. The wrapper-based feature selection methods input the subsets of features into some certain learning algorithms until obtaining the best performance.   Dy and Brodley used the expectation-maximization clustering algorithm and two different metrics (scatter separability and maximum likelihood) for the feature search and selection~\cite{Dy2004Feature}. Based on the support vector machines with kernel functions, Maldonado et al. proposed a wrapper-based method to choose features in terms of the number of errors in the validation subset~\cite{Maldonado2009wrapper}. Since the wrapper-based method need use the learning algorithm to train each selected feature, it will result in the expensive time cost and could not be applied in large-scale machine learning task. The embedded-based feature selection methods incorporate the process of  feature selection into some certain learning tasks, such as clustering and dimension reduction. Hou et al. presented a novel framework for unsupervised feature selection by combining the embedding learning and sparse regression~\cite{Hou2013Joint}. Liu et al. employed the locally linear embedding method to obtain feature weight matrix and imposed $\ell_{1}$-norm on the reconstruction term to select the robust features~\cite{Liu2020Robust}. Shi et al. embedded the unsupervised feature selection into a robust spectral regression model with sparse constraint~\cite{RSFS}. The aforementioned single-view unsupervised feature selection methods deem that different views' features are independent and ignore the underlying relations among different views. They could not effectively cope with the feature selection of unlabeled multi-view data.

Different from the first kind of MUFS method by concatenating different views' features together, the other one for multi-view unsupervised feature selection is to deal with multi-view data directly and discuss the interaction information among multiple views. Hou et al. presented a novel multi-view feature selection framework by adaptively learning the common similarity matrix across different views, rather than using the fixed  graph structure~\cite{Hou:Multi}.  By the weight combination of different views' similarity matrices, Dong et al. proposed a multi-view feature selection method by learning a collaborative similarity structure and integrating it into a sparse regression model~\cite{ACSL}. Tan et al. embedded the feature selection into a clustering process by learning a consensus indicator matrix and fusing the indicator matrices of different views with adaptive view-weights~\cite{CGMV-UFS}. Wan et al. presented an embedding multi-view unsupervised feature selection method by simultaneously learning the local and global structures among different views~\cite{ASE}. Zhao et al. integrated the process of feature selection into the multi-non-negative matrix factorization for learning both the view-specific features and the common features of different views~\cite{CoUFC}. However, the aforementioned  methods have a common assumption that all views are complete, which results in that they could not be directly applied in the incomplete multi-view unlabeled data. Moreover, the data will change over time in the dynamic environment. These methods based on the full batch manner will suffer from the high storage cost and expensive computation time. Although OMVFS~\cite{OMVFS} presents an incremental method for feature selection on multi-view data with complete views, it is still confronted with at least two issues, i.e., it is not suitable for incomplete multi-view data and the importance of  different views does not take into account. Hence, by the utilization of the consistent and complementary information among different views, we propose a novel multi-view unsupervised feature selection method I$^2$MUFS for the incomplete multi-view streaming data. Besides, we also provide an efficient method to solve the proposed objective function by introducing the incremental mechanisms.

\section{The Proposed Method}
\subsection{Notations and Definitions}
Throughout the paper, all the matrices are denoted as boldface capital letters. For a matrix $\textbf{V} \in R^{d \times k}$,  $\textbf{V}_{ij}$ denotes its $(i,j)$-th entry, $\textbf{V}(i, :)$ and $\textbf{V}(:, j)$ denote its $i$-th row and $j$-th column, respectively. $Tr(\textbf{V})$ is the trace of $\textbf{V}$ and $\textbf{V}^{T}$ is the transpose of $\textbf{V}$. The Frobenius norm of $\textbf{V}$ is defined as $\sqrt{\sum_{i=1}^{d} \sum_{j=1}^{k} \textbf{V}_{ij}^{2}}$. $\|\textbf{V}\|_{2,1}=\sum_{i=1}^{d} \sqrt {\sum_{j=1}^{k} \textbf{V}_{ij}^{2}}$ represents the $\ell_{2,1}$ norm of $\textbf{V}$.

In the dynamic data environment, the data arrive chunk by chunk. Suppose that the incomplete multi-view data $\textbf{X}_t=[\textbf{X}_t^{(1)};\textbf{X}_t^{(2)};\cdots;\textbf{X}_t^{(n_v)}]\in R_{+}^{d \times N_{t}}$ arrive at time $t$, where $\textbf{X}_t^{(v)} \in R_{+}^{d_{v} \times N_{t}}$ denotes the $v$-th view data, $d_v$ is the dimensionality of $v$-th view such that $\sum\limits_{v=1}^{{{n}_{v}}}{{{d}_{v}}}=d$ and $N_t$ is the number of instances in the $t$-th chunk. The objective of incomplete multi-view feature selection is to select the discriminative features from the arrived data $[\textbf{X}_1,\textbf{X}_2,\cdots,\textbf{X}_T]$ up to time $T$.

\subsection{Preliminaries}
Non-negative Matrix Factorization (NMF) is widely used in unsupervised learning including dimension reduction and clustering~\cite{NIPS2000nmf}. Let $\textbf{X}\in R_{+}^{d \times N}$ denote the input non-negative data matrix, where each column indicates an instance with the $d$-dimensional features. NMF  aims to approximately factorize the data matrix $\textbf{X}$  into two low-rank non-negative matrices, which can be formulated as follows:
\begin{align}\label{EqNMF}
&\min_{\textbf{U},\textbf{V}} \| (\textbf{X}-\textbf{V}{\textbf{U}}^{T})  \|_{F}^{2} \notag \\
&s.t.~\textbf{V}\ge 0, \textbf{U}\ge 0,
\end{align}
where $\textbf{V} \in R_{+}^{d\times K}$ and $\textbf{U} \in R_{+}^{N \times K}$ denote the latent feature matrix and the clustering indicator matrix, respectively. Eq.~(\ref{EqNMF}) assumes the input data matrix is complete with no missing values. It could not be directly applied in the incomplete data. To cope with this issue, Weighted Non-negative Matrix Factorization (WNMF) is presented by introducing a weighted matrix~\cite{WNMF}. The formulation of WNMF can be described as follows:
\begin{align}\label{EqWNMF}
&\min_{\textbf{U},\textbf{V}} \| (\textbf{X}-\textbf{V}{\textbf{U}^{T})\textbf{W}}  \|_{F}^{2} \notag \\
&s.t.~\textbf{V}\ge 0, \textbf{U}\ge 0,
\end{align}
where $\textbf{W}$ is the weighted matrix and the corresponding matrix entry $\textbf{W}_{ij}=0$, if the entry  $\textbf{X}_{ij}$ is missing, otherwise $\textbf{W}_{ij}=1$. However, it cannot be directly used for the dynamic incomplete multi-view data.  In what follows, we extend WNMF for the feature selection of dynamic incomplete multi-view data and integrate the complementary and consensus information of different views into the extended model.
\subsection{Formulation of I$^2$MUFS}

In the dynamic data environment, assuming that $N_{t}$ incomplete multi-view data samples with $n_v$ views $\{\textbf{X}_t^{(v)} \in R_{+}^{d_{v} \times N_{t}}, v=1,2,\cdots, n_v \}$ arrived at time $t$, where $d_{v}$ represents the feature dimension of the $v$-th view. For each view data matrix $\textbf{X}_t^{(v)}$, we employ WNMF to factorize it, which can be described as follows:
\begin{align}\label{Eq1}
&\min_{\textbf{U}_{t}^{(v)},\textbf{V}^{(v)}}\sum_{v=1}^{n_v}  \| (\textbf{X}_{t}^{(v)}-\textbf{V}^{(v)}{\textbf{U}_{t}^{(v)}}^{T})\textbf{W}_{t}^{(v)}  \|_{F}^{2} \notag \\
&s.t.~\textbf{V}^{(v)}\ge 0, \textbf{U}_t^{(v)}\ge 0,
\end{align}
where $\textbf{V}^{(v)} \in R_{+}^{d_v \times K}$ denotes the latent feature matrix of the $v$-th view, $\textbf{U}_{t}^{(v)} \in R_{+}^{N_t \times K}$ represents the clustering indicator matrix of the $v$-th view and $\textbf{W}_{t}^{(v)} \in R^{N_t \times N_t} $ is a diagonal weight matrix, where the diagonal entry $\textbf{W}_{j_{t}j_{t}}^{(v)}$ describe the weight of the instance $j_{t}$ belonging to the view $v$. In the classical WNMF~\cite{WNMF}, the diagonal entry of $\textbf{W}_{t}^{(v)}$ equals to 0, if the instance is missing, 1 otherwise. This is equivalent to the missing values are filled with 0, which is unreasonable in some real applications. A reasonable method is using the mean values to fill missing values~\cite{de2011handbook}. Since the data arrive consecutively, we could not use the mean value of all data. In what follows, we introduce a weighted filling method by utilization of the available information at present.

To describe the missing instances of different views at time $t$, an indicator matrix $\textbf{M}^{t}\in R^{N_t \times n_v}$ is defined as follows:
\begin{equation}\label{Eq2}
\textbf{M}_{iv}^{t}=
\begin{cases}
1& \text{if the instance $i$ is in the $v$-th view at time $t$;}\\
0& \text{otherwise.}
\end{cases}
\end{equation}
Let $\textbf{X}_{t}^{(v)}(:, j_t)$ indicate the newest arriving instance. If it is missing in the $v$-th view, it is imputed by  the weighted average of up-to-date instances:
\begin{equation}\label{Eq3}
\textbf{X}_{t}^{(v)}(:, j_t)= \sum_{i=1}^{j_t} \frac{\textbf{M}_{iv}^{t}}{n^{v}_{j_{t}}}\textbf{X}_{t}^{(v)}(:, i),
\end{equation}
where $n^{v}_{j_{t}}=\sum_{i=1}^{j_t}\textbf{M}_{iv}^{t}$ indicates the number of the available instances in view $v$ up to now. Moreover, the diagonal entry $\textbf{W}_{j_{t}j_{t}}^{(v)}$ of the weight matrix $\textbf{W}_{t}^{(v)}$ can be computed by following equation:
\begin{equation}\label{Eq4}
\textbf{W}_{j_{t}j_{t}}^{(v)}=
\begin{cases}
1& \text{if  $\textbf{X}_{t}^{(v)}(:, j_t)$ is in the $v$-th view;}\\
n^{v}_{j_{t}}/j_t &  \text{otherwise.}
\end{cases}
,
\end{equation}
which measures the quality of available information until now.

For incomplete multi-view streaming data, Eq.~(\ref{Eq1}) only considers the feature selection of the current data chunk and deem all views equally important. The goal of incomplete multi-view feature selection is to find the feature selection matrix $\textbf{V}^{(v)}$ from the arrived data $[\textbf{X}_1,\textbf{X}_2,\cdots,\textbf{X}_T]$ up to time $T$ by employing the complementary and consensus information of different views. In order to exploit the complementary information, we rewrite the Eq.~(\ref{Eq1}) as follows by selecting the feature matrix  $\textbf{V}^{(v)}$ of each view from all arrived data and fusing them with adaptive view weights to measure the importance of different views.
\begin{align}\label{Eq5}
&\min_{\textbf{U}_{t}^{(v)},\textbf{V}^{(v)},\alpha^{(v)}}\sum_{v=1}^{n_v} (\alpha^{(v)})^{\lambda} \sum_{t=1}^{T}   \| (\textbf{X}_{t}^{(v)}-\textbf{V}^{(v)}{\textbf{U}_{t}^{(v)}}^{T})\textbf{W}_{t}^{(v)}  \|_{F}^{2}  \notag +\eta^{(v)} \| \textbf{V}^{(v)}  \|_{2,1} \notag \\
&s.t. \textbf{U}_t^{(v)}\ge 0, \textbf{V}^{(v)}\ge 0, {\textbf{U}_t^{(v)}}^{T}\textbf{U}_t^{(v)}=\textbf{I},\alpha^{(v)} \ge 0, \sum_{v=1}^{n_v}\alpha^{(v)}=1
\end{align}
where  $\alpha^{(v)}$ is the adaptive weight of the $v$-th view, $\lambda$ is a parameter used to control the smoothness of the weights distribution and the orthogonal constraint of $\textbf{U}_t^{(v)}$ is the relaxed form of the discrete constraint of the clustering indicator matrix $\textbf{U}_{t}^{(v)} \in \{0,1\}^{N_t \times K}$. By imposing row sparsity on $\textbf{V}^{(v)}$ through $\ell_{2,1}$-norm, we can select the features according to the row norms of $\textbf{V}^{(v)}$ in a descending order~\cite{wang2015embedded}.

In addition, considering the consistent information across different views, the clustering indicator matrix $\textbf{U}_t^{(v)}$ of each view  should be closed to the consensus clustering indicator matrix $\textbf{U}_t^{*}$, which is described by the following term.
\begin{align}\label{Eq6}
\min_{\textbf{U}_t^{*}, \textbf{U}_t^{(v)}}  \| \textbf{W}_{t}^{(v)}(\textbf{U}_{t}^{(v)}-\textbf{U}_{t}^{*})  \|_{F}^{2}
\end{align}
Furthermore, motivated by the similar instances should have the similar cluster labels, the following term based on the spectral analysis~\cite{li2013clustering} is appended to the proposed model.
\begin{align}\label{Eq7}
\min_{\textbf{U}_t^{(v)}}Tr({\textbf{U}_t^{(v)}}^{T}\textbf{L}_t^{(v)}\textbf{U}_t^{(v)}),
\end{align}
where $\textbf{L}_t^{(v)}=\textbf{D}_t^{(v)}-\textbf{S}_t^{(v)}$ is the Laplacian matrix of the $v$-th view in the $t$-th data chunk. $\textbf{S}_t^{(v)}$ and $\textbf{D}_t^{(v)}$ denote the similarity matrix and the corresponding degree matrix, respectively.

By putting Eqs.~(\ref{Eq5}), (\ref{Eq6}) and (\ref{Eq7}) together, we have the overall objective function of the proposed method as follows:
\begin{align}\label{Eq8}
&\min_{\textbf{U}_{t}^{(v)}, \textbf{U}_{t}^{*}, \textbf{V}^{(v)} , \alpha^{(v)}}\sum_{v=1}^{n_v}\sum_{t=1}^{T} (  (\alpha^{(v)})^{\lambda} \| (\textbf{X}_{t}^{(v)}-\textbf{V}^{(v)}{\textbf{U}_{t}^{(v)}}^{T})\textbf{W}_{t}^{(v)}  \|_{F}^{2}  +\beta^{(v)} \| \textbf{W}_{t}^{(v)}(\textbf{U}_{t}^{(v)}-\textbf{U}_{t}^{*})  \|_{F}^{2} +
\theta^{(v)}Tr({\textbf{U}_{t}^{(v)}}^{T}\textbf{L}_{t}^{(v)}\textbf{U}_{t}^{(v)}) )  +\eta^{(v)} \| \textbf{V}^{(v)}  \|_{2,1} \notag   \\
&s.t. \textbf{U}_{t}^{(v)}\ge 0, \textbf{V}^{(v)}\ge 0, \textbf{U}_{t}^{*} \ge 0, {\textbf{U}_t^{(v)}}^{T}\textbf{U}_t^{(v)}=\textbf{I} ,\alpha^{(v)} \ge 0, \sum_{v=1}^{n_v}\alpha^{(v)}=1,
\end{align}
where $\beta^{(v)} \ge 0$ and $\theta^{(v)} \ge 0$ are the trade-off hyper-parameters.

As can be seen from Eq.~(\ref{Eq8}), our model employs the extended WNMF to learn different latent feature matrices  across different incomplete views and fuse them with adaptive view weights, which both consider the view-specific features and the complementary information of different views. It is better than learning the unified feature matrix with fixed weights. Besides, by considering  the consensus information of different views, the proposed model learns a consensus clustering indicator matrix and preserves the local geometrical structure of data based on the spectral analysis. These two parts can promote each other to improve the performance of unsupervised feature selection on incomplete multi-view data.

\section{Optimization Procedure}
In Eq.~(\ref{Eq8}), the objective function is not convex to all variables simultaneously. In this section, we introduce an alternative iterative algorithm to solve this problem by optimizing one variable and fixing the rest one. Moreover, an incremental scheme is proposed to improve the computational efficiency.

\subsection{Optimize $\textbf{V}^{(v)}$ with Other Variables Fixed}
While $\textbf{U}_{t}^{(v)}$, $\textbf{U}_t^{*}$ and $\alpha^{(v)}$ are fixed at current time $t$, the optimization of $\textbf{V}^{(v)}$ is independent for different $v$. Then, the optimization problem in Eq.~(\ref{Eq8}) is reduced as follows:
\begin{align}
&\min_{\textbf{V}^{(v)}} (\alpha^{(v)})^{\lambda} \| (\textbf{X}_{t}^{(v)}-\textbf{V}^{(v)}{\textbf{U}_{t}^{(v)}}^{T})\textbf{W}_{t}^{(v)}  \|_{F}^{2} +\eta^{(v)} \| \textbf{V}^{(v)}  \|_{2,1} \notag \\
&s.t. \textbf{V}^{(v)}\ge 0.\label{NewEq1}
\end{align}

By means of  the Lagrange multiplier method, we can obtain the following Lagrange function:
\begin{align}\label{Eq9}
&J(\textbf{V}^{(v)}) = (\alpha^{(v)})^{\lambda} \sum_{i=1}^{t}  \| (\textbf{X}_{i}^{(v)}-\textbf{V}^{(v)}{\textbf{U}_{i}^{(v)}}^{T})\textbf{W}_{i}^{(v)}  \|_{F}^{2}+\eta^{(v)} \| \textbf{V}^{(v)}  \|_{2,1}+Tr(\bm{\Phi}^T{\textbf{V}^{(v)}}),
\end{align}
where $\bm{\Phi}$ is the Lagrange multiplier matrix for the constraint $\textbf{V}^{(v)}\ge 0$. Taking the derivative of $J(\textbf{V}^{(v)})$ w.r.t. $\textbf{V}^{(v)}$, we have
\begin{align}
&\frac{\partial J(\textbf{V}^{(v)})}{\partial \textbf{V}^{(v)}} = (\alpha^{(v)})^{\lambda}(2\textbf{V}^{(v)} \sum_{i=1}^{t}{\textbf{U}_{i}^{(v)}}^{T}\tilde{\textbf{W}}_{i}^{(v)}\textbf{U}_{i}^{(v)}-2\sum_{i=1}^{t}\textbf{X}_{i}^{(v)}\tilde{\textbf{W}}_{i}^{(v)}\textbf{U}_{i}^{(v)} )+\eta^{(v)}\textbf{H}^{(v)}\textbf{V}^{(v)}+\bm{\Phi},
\end{align}
where $\textbf{H}^{(v)} = diag(\frac{1}{ \| \textbf{V}^{(v)}(d_{H},:) \|_{2}+\varepsilon}),$ $\varepsilon$ is a small positive value, $d_{H}=1, 2, ..., d_{v}$ and $\tilde{\textbf{W}}_{i}^{(v)} = \textbf{W}_{i}^{(v)}{\textbf{W}_{i}^{(v)}}^{T}$.

For the convenience of following discussion, let $\textbf{R}_{t}^{(v)} \triangleq \sum_{i=1}^{t}{\textbf{U}_{i}^{(v)}}^{T}\tilde{\textbf{W}}_{i}^{(v)}\textbf{U}_{i}^{(v)}$ and $\textbf{Q}_{t}^{(v)} \triangleq \sum_{i=1}^{t}\textbf{X}_{i}^{(v)}\tilde{\textbf{W}}_{i}^{(v)}\textbf{U}_{i}^{(v)}$. Then, while a new chunk arrives, $\textbf{R}_{t}^{(v)}$ and $\textbf{Q}_{t}^{(v)}$ can be computed as follows:
\begin{equation}\label{Eq11}
\begin{aligned}
\textbf{R}_{t}^{(v)} &= \textbf{R}_{t-1}^{(v)} + {\textbf{U}_{t}^{(v)}}^{T}\tilde{\textbf{W}}_{t}^{(v)}\textbf{U}_{t}^{(v)} \\
\textbf{Q}_{t}^{(v)} &= \textbf{Q}_{t-1}^{(v)} +\textbf{X}_{t}^{(v)}\tilde{\textbf{W}}_{t}^{(v)}\textbf{U}_{t}^{(v)}
\end{aligned}
\end{equation}
By using the previous computational results, Eq.~(\ref{Eq11}) can incrementally update $\textbf{R}_{t}^{(v)}$ and $\textbf{Q}_{t}^{(v)}$ without recomputing in the whole data set from scratch.

According to the Karush-Kuhn-Tucker (KKT) complementary condition $\bm{\Phi}_{rs}\textbf{V}_{rs}^{(v)}=0$, $\textbf{V}_{rs}^{(v)}$ can be updated by following rule:
\begin{equation}\label{Eq12}
\begin{aligned}
\textbf{V}_{rs}^{(v)} \gets \textbf{V}_{rs}^{(v)}\sqrt{\frac{2(\alpha^{(v)})^{\lambda} [\textbf{Q}_{t}^{(v)}]_{rs}}{[2(\alpha^{(v)})^{\lambda}\textbf{V}^{(v)}\textbf{R}_{t}^{(v)}+\eta^{(v)}\textbf{H}^{(v)}\textbf{V}^{(v)}]_{rs}} }
\end{aligned}
\end{equation}

\subsection{Optimize $\textbf{U}_{t}^{(v)}$ with Other Variables Fixed}
The optimization of $\textbf{U}_{t}^{(v)}$ is independent for different $v$ while $\textbf{V}^{(v)}$, $\textbf{U}_t^{*}$ and $\alpha^{(v)}$ are fixed. Hence, $\textbf{U}_{t}^{(v)}$ can be updated by solving the following optimization problem:
\begin{align}
&\min_{\textbf{U}_{t}^{(v)}}(\alpha^{(v)})^{\lambda} \| (\textbf{X}_{t}^{(v)}-\textbf{V}^{(v)}{\textbf{U}_{t}^{(v)}}^{T})\textbf{W}_{t}^{(v)}  \|_{F}^{2} +\beta^{(v)} \| \textbf{W}_{t}^{(v)}(\textbf{U}_{t}^{(v)}-\textbf{U}_{t}^{*})  \|_{F}^{2} +
\theta^{(v)}Tr({\textbf{U}_{t}^{(v)}}^{T}\textbf{L}_{t}^{(v)}\textbf{U}_{t}^{(v)})  \notag   \\
&s.t. \textbf{U}_{t}^{(v)}\ge 0, {\textbf{U}_t^{(v)}}^{T}\textbf{U}_t^{(v)}=\textbf{I}.\label{NewEq2}
\end{align}

Then, we further construct the following Lagrange function:
\begin{align}
J(\textbf{U}_t^{(v)}) =&(\alpha^{(v)})^{\lambda}  \| (\textbf{X}_{t}^{(v)}-\textbf{V}^{(v)}{\textbf{U}_{t}^{(v)}}^{T})\textbf{W}_{t}^{(v)}  \|_{F}^{2} +\beta^{(v)} \| \textbf{W}_{t}^{(v)}(\textbf{U}_{t}^{(v)}-\textbf{U}_{t}^{*})  \|_{F}^{2}  +\theta^{(v)}Tr({\textbf{U}_{t}^{(v)}}^{T}\textbf{L}_{t}^{(v)}\textbf{U}_{t}^{(v)}) \notag \\
&+\xi^{(v)}  \| {\textbf{U}_{t}^{(v)}}^{T}\textbf{U}_{t}^{(v)}-\textbf{I}  \|_{F}^{2} + Tr(\bm{\Psi}^{T}{\textbf{U}_t^{(v)}}),
\end{align}
where $\xi^{(v)}$ and $\bm{\Psi}$ are the Lagrange multipliers for the orthogonal constraint and non-negative constraint, respectively. To guarantee the orthogonality satisfied, $\xi^{(v)}$ is set to the large constant in practice.

Then the partial derivative of $J(\textbf{U}_t^{(v)})$ w.r.t. $\textbf{U}_{t}^{(v)}$ is
\begin{align}
\frac{\partial J}{\partial \textbf{U}_{t}^{(v)}}  = & (\alpha^{(v)})^{\lambda}(2\tilde{\textbf{W}}_{t}^{(v)}\textbf{U}_{t}^{(v)}{\textbf{V}^{(v)}}^T\textbf{V}^{(v)}-2\tilde{\textbf{W}}_{t}^{(v)}{\textbf{X}_{t}^{(v)}}^{T}\textbf{V}^{(v)} ) +2\beta^{(v)}\tilde{\textbf{W}}_{t}^{(v)}(\textbf{U}_{t}^{(t)}-\textbf{U}_{t}^{*}) +2\theta^{(v)}\textbf{L}_{t}^{(v)}\textbf{U}_{t}^{(v)}  \notag \\
&+ 4\xi^{(v)}(\textbf{U}_{t}^{(v)}{\textbf{U}_{t}^{(v)}}^{T}\textbf{U}_{t}^{(v)}-\textbf{U}_{t}^{(v)})+\bm{\Psi}
\end{align}
By the utilization of KKT complementary condition $\bm{\Psi}_{rs}[\textbf{U}_{t}^{(v)}]_{rs}=0$, we can obtain the following updating rule:
\begin{align}\label{Eq15}
[\textbf{U}_{t}^{(v)}]_{rs}  \gets  [\textbf{U}_{t}^{(v)}]_{rs}\sqrt{\frac{[\textbf{G}_{t}^{(v)}]_{rs}}
{[\textbf{P}_{t}^{(v)}]_{rs}}}~,
\end{align}
where
\begin{align}
\textbf{G}_{t}^{(v)} &= (\alpha^{(v)})^{\lambda}{\tilde{\textbf{W}}_{t}^{(v)}\textbf{X}_{t}^{(v)}}^{T}\textbf{V}^{(v)}+\beta^{(v)}\tilde{\textbf{W}}_{t}^{(v)}\textbf{U}_{t}^{*} +2\xi^{(v)}\textbf{U}_{t}^{(v)}+\textbf{Z}_{t}^{-}\textbf{U}_{t}^{(v)},
\end{align}
\begin{align}
\textbf{P}_{t}^{(v)} &=(\alpha^{(v)})^{\lambda}\tilde{\textbf{W}}_{t}^{(v)}\textbf{U}_{t}^{(v)}{\textbf{V}^{(v)}}^{T}\textbf{V}^{(v)} + \beta^{(v)}\tilde{\textbf{W}}_{t}^{(v)}\textbf{U}_{t}^{(v)}+\textbf{Z}_{t}^{+}\textbf{U}_{t}^{(v)}+2\xi^{(v)}\textbf{U}_{t}^{(v)}{\textbf{U}_{t}^{(v)}}^{T}\textbf{U}_{t}^{(v)},
\end{align}
and
\begin{align}
&\textbf{Z}_{t}=\textbf{Z}_{t}^{+}-\textbf{Z}_{t}^{-}=\theta^{(v)}\textbf{L}_{t}^{(v)}, [\textbf{Z}_{t}^{+}]_{rs}=\frac{1}{2}(|[\textbf{Z}_{t}]_{rs}|+[\textbf{Z}_{t}]_{rs}),[\textbf{Z}_{t}^{-}]_{rs}=\frac{1}{2}(|[\textbf{Z}_{t}]_{rs}|-[\textbf{Z}_{t}]_{rs}).
\end{align}

\subsection{Optimize $\textbf{U}_t^{*}$ with Other Variables Fixed}
While $\textbf{U}_{t}^{(v)}$, $\textbf{V}^{(v)}$ and $\alpha^{(v)}$ are fixed, the optimization problem in Eq.~(\ref{Eq8}) becomes
\begin{align}
&\min_{\textbf{U}_{t}^{*}}\sum_{v=1}^{n_v}(  \beta^{(v)} \| \textbf{W}_{t}^{(v)}(\textbf{U}_{t}^{(v)}-\textbf{U}_{t}^{*})  \|_{F}^{2}~s.t.~\textbf{U}_{t}^{*} \ge 0.\label{NewEq3}
\end{align}

Let $J(\textbf{U}_t^{*})=\sum_{v=1}^{n_v} \beta^{(v)} \| \textbf{W}_{t}^{(v)}(\textbf{U}_{t}^{(v)}-\textbf{U}_{t}^{*}) \|_{F}^{2}$. Taking the derivative of $J(\textbf{U}_t^{*})$ w.r.t $\textbf{U}_t^{*}$ and setting it to 0, we can obtain the solution for $\textbf{U}_t^{*}$ as follows:
\begin{align}\label{Eq20}
\textbf{U}_{t}^* = (\sum_{v=1}^{n_{v}}\beta^{(v)}\tilde{\textbf{W}}_{t}^{(v)})^{-1}\sum_{v=1}^{n_{v}}\beta^{(v)}\tilde{\textbf{W}}_{t}^{(v)}\textbf{U}_{t}^{(v)},
\end{align}
which satisfies the non-negative constraint.
\subsection{Optimize $\alpha^{(v)}$ with Other Variables Fixed}
While $\textbf{V}^{(v)}$, $\textbf{U}_{t}^{(v)}$ and $\textbf{U}_t^{*}$ are fixed, we can update $\alpha^{(v)}$ by solving the following problem:
\begin{align}
&\min_{\alpha^{(v)}}\sum_{v=1}^{n_v} \sum_{t=1}^{T} (\alpha^{(v)})^{\lambda} \| (\textbf{X}_{t}^{(v)}-\textbf{V}^{(v)}{\textbf{U}_{t}^{(v)}}^{T})\textbf{W}_{t}^{(v)}  \|_{F}^{2}  \notag \\
&s.t. \alpha^{(v)} \ge 0, \sum_{v=1}^{n_v}\alpha^{(v)}=1.\label{NewEq4}
\end{align}

Let $L_{T}^{(v)}\triangleq \sum_{t=1}^{T} \| (\textbf{X}_{t}^{(v)}-\textbf{V}^{(v)}{\textbf{U}_{t}^{(v)}}^{T})\textbf{W}_{t}^{(v)}  \|_{F}^{2}$. Then, $L_{T}^{(v)}=L_{T-1}^{(v)}+\| (\textbf{X}_{t}^{(v)}-\textbf{V}^{(v)}{\textbf{U}_{t}^{(v)}}^{T})\textbf{W}_{t}^{(v)}  \|_{F}^{2}$. Namely, $L_{T}^{(v)}$ can be computed incrementally. If we first ignore the constraint $\alpha^{(v)} \ge 0$, the Lagrange function is constructed as follows:
\begin{align}
&J(\alpha) =\sum_{v=1}^{n_v} (  (\alpha^{(v)})^{\lambda}L_{T}^{(v)}-\varphi(\sum_{v=1}^{n_v}\alpha^{(v)}-1),
\end{align}
where $\varphi$ is the Lagrange multiplier.

Taking the derivative of $J(\alpha)$ w.r.t. $\alpha^{(v)}$ and setting it to zero, it is easy to obtain
\begin{equation}\label{NewEq5}
\lambda(\alpha^{(v)})^{\lambda-1}L_{T}^{(v)}-\varphi=0.
\end{equation}

Then, we have $\alpha^{(v)}=(\frac{\varphi}{\lambda L_{T}^{(v)}})^{\frac{1}{\lambda-1}}$. Since $\sum_{v=1}^{n_v}\alpha^{(v)}=1$, we can obtain the solution for $\alpha^{(v)}$ as follows:
\begin{align}\label{Eq21}
\alpha^{(v)} = \frac{(L_{T}^{(v)})^{\frac{1}{1-\lambda}}}{\sum_{v=1}^{n_v}(L_{T}^{(v)})^{\frac{1}{1-\lambda}}}.
\end{align}
Since $L_{T}^{(v)}>0$, it is easy to verify that the solution in Eq.~\ref{Eq21} satisfies the non-negative constraint.

The overall optimization procedure for I$^2$MUFS is summarized in Algorithm 1.

\begin{algorithm}[h]
  \caption{Iterative algorithm of I$^2$MUFS}
  \label{alg:algorithm}
  \textbf{Input}: The incomplete multi-view data chunks $\textbf{X}_{t}^{(v)}~(t=1,2,\cdots,T; v=1,2,\cdots,n_{v})$, the number of clusters $K$, parameters $\lambda$,
  $\beta^{(v)}$, $\theta^{(v)}$ and $\eta^{(v)}~(v=1,2,\cdots,n_{v})$. \\
  \textbf{Output}: Feature selection matrix $\textbf{V}^{(v)}~(v=1,2,\cdots,n_{v})$.
  \begin{algorithmic}[1]
    \State \textbf{Initialization}: Randomly initilize $\textbf{V}^{(v)}$, $\alpha^{(v)}=\frac{1}{n_{v}}$, $\textbf{R}_{0}^{(v)}=0$, $\textbf{Q}_{0}^{(v)}=0$ for each view.	
    \For{$t=1$ to $T$}
     	 \For{$v=1$ to $n_{v}$}
	 	\State Fill in missing values of $\textbf{X}_{t}^{(v)}$ according to Eq.~(\ref{Eq3}).
		\State Construct weight matrix $\textbf{W}_{t}^{(v)}$ based on Eq.~(\ref{Eq4}).
	        \State Initialize $\textbf{U}_{t}^{(v)}$.
	\EndFor
		
    \Repeat
    	\For {$v=1$ to $n_{v}$}
      		\State Update $\textbf{V}^{(v)}$ according to Eq.~(\ref{Eq12}).
      		\State Update $\textbf{U}_{t}^{(v)}$ according to Eq.~(\ref{Eq15}).
      	\EndFor
       \State Compute $\textbf{U}_{t}^{*}$ according to Eq.~(\ref{Eq20}).
       \State Compute $\alpha^{(v)}$ in terms of Eq.~(\ref{Eq21}).
     \Until{\textit{Convergence}}
      \State Incremental update $\textbf{R}_{t}^{(v)}$ and $\textbf{Q}_{t}^{(v)}$ based on Eq.~(\ref{Eq11}).
      \State $\textbf{R}_{t}^{(v)} = \textbf{R}_{t-1}^{(v)} + {\textbf{U}_{t}^{(v)}}^{T}\tilde{\textbf{W}}_{t}^{(v)}\textbf{U}_{t}^{(v)}$
	  \State $\textbf{Q}_{t}^{(v)} = \textbf{M}_{t-1}^{(v)} +\textbf{X}_{t}^{(v)}\tilde{\textbf{W}}_{t}^{(v)}\textbf{U}_{t}^{(v)}$
    \EndFor
   \State Rank $\|\textbf{V}^{(v)}(i,:)\|_{2}~(i=1,2,\cdots,d_{v})$ in a descending order and select the top $l$ features.
  \end{algorithmic}
\end{algorithm}

\section{Theoretical Analysis of I$^2$MUFS}
In this section, we give the theoretical analysis of the proposed algorithm including the convergence and  time complexity analysis.
\subsection{Convergence Analysis}
As shown in the optimization procedure, the convergence of Algorithm 1 depends on four subproblems. If we can  prove the convergence of updating $\textbf{V}^{(v)}$, $\textbf{U}_t^{(v)}$, $\textbf{U}_t^{*}$, $\alpha^{(v)}$ and $\textbf{W}_t^{(v)}$ in the $v$-th view at the $t$-th chunk, then  it can be proved by the similar way on other views and chunks. For brevity, we denote $\textbf{V}^{(v)}$, $\textbf{U}_t^{(v)}$, $\textbf{U}_t^{*}$, $\alpha^{(v)}$ and $\textbf{W}_t^{(v)}$ by $\textbf{V}$, $\textbf{U}$, $\textbf{U}^{*}$, $\alpha$ and $\textbf{W}$, respectively. The convergence of $\textbf{U}^*$ and $\alpha$ can be guaranteed by their closed-form solutions in Eqs.~(\ref{Eq20}) and (\ref{Eq21}). Here, we give the convergence proof of updating $\textbf{V}$. For the convergence of updating $\textbf{U}$, we can prove it in a similar way.

In order to prove the objective function value of problem~(\ref{NewEq1}) monotonically decreases by Algorithm 1 in each iteration, we introduce the following definition and lemma.

\begin{definition}\label{Def1}~\cite{lee2001algorithms}
	$Z(v, v')$ is an auxiliary function of $F(v)$ if satisfying the following conditions \\
	\begin{equation}
	Z(v, v') \ge F(v) , Z(v, v) = F(v).
	\end{equation}
\end{definition}

\begin{lemma}\label{Lemma1}~\cite{lee2001algorithms}
	If $Z(v, v')$ is an auxiliary function of $F(v)$, then $F(v)$ is non-increasing under the following update rule:\\
	\begin{equation}\label{update}
	 v ^{(\tau +1)} = \mathop{\arg\min}_{v} Z(v, v^{(\tau)})
	 \end{equation}
\end{lemma}

Let $\textbf{H} \triangleq diag(\frac{1}{ \| \textbf{V}(d_{H},:) \|_{2}+\varepsilon}),  d_{H}=1, 2, ..., d.$. Then, according to~\cite{li2018generalized}, both $\| \textbf{V} \|_{2,1}$ and $\frac{1}{2}Tr(\textbf{V}^{T}\textbf{H}\textbf{V})$ satisfy the same KKT conditions in the associated Lagrange function of problem~(\ref{NewEq1}). Hence, problem~(\ref{NewEq1}) can be transformed as follows:
\begin{align}\label{Probtrans}
\min_{\textbf{V}\geqslant 0}\alpha^{\lambda}  \| (\textbf{X}-\textbf{V}\textbf{U}^{T})\textbf{W} \|_{F}^{2}+\frac{\eta}{2}Tr(\textbf{V}^{T}\textbf{H}\textbf{V}).
\end{align}
Notice that the superscript $v$ and the subscript $t$ in problem~(\ref{Probtrans}) are ignored for brevity. It is easy to see that
\begin{align}\label{Probdeduction}
&\min_{\textbf{V}\geqslant 0}\alpha^{\lambda}  \| (\textbf{X}-\textbf{V}\textbf{U}^{T})\textbf{W} \|_{F}^{2}+\frac{\eta}{2}Tr(\textbf{V}^{T}\textbf{H}\textbf{V}) \notag \\
& \Leftrightarrow \min_{\textbf{V}\geqslant 0}\alpha^{\lambda} Tr(\textbf{V}^{T}\textbf{V}\textbf{U}^{T}\tilde{\textbf{W}}\textbf{U}-2\textbf{V}^{T}\textbf{X}\tilde{\textbf{W}}\textbf{U})+\frac{\eta}{2}Tr(\textbf{V}^{T}\textbf{H}\textbf{V}) \notag \\
& \Leftrightarrow \min_{\textbf{V}\geqslant 0}\alpha^{\lambda} Tr(\textbf{V}^{T}\textbf{V}\textbf{R} -2\textbf{V}^{T}\textbf{Q} )+\frac{\eta}{2}Tr(\textbf{V}^{T}\textbf{H}\textbf{V})
\end{align}
where $\textbf{Q}=\textbf{X}\tilde{\textbf{W}}\textbf{U}, \textbf{R}=\textbf{U}^{T}\tilde{\textbf{W}}\textbf{U}$ and $\tilde{\textbf{W}} = \textbf{W}\textbf{W}^{T}$. Namely, problem~(\ref{Probtrans}) is equal to problem~(\ref{Probdeduction}). Let  $J(\textbf{V})=\alpha^{\lambda} Tr(\textbf{V}^{T}\textbf{V}\textbf{R} -2\textbf{V}^{T}\textbf{Q} )+\frac{\eta}{2}Tr(\textbf{V}^{T}\textbf{H}\textbf{V})$ denote the objective function of problem~(\ref{Probdeduction}). We give the auxiliary function of $J(\textbf{V})$ by the following Lemma.

\begin{lemma}\label{Lemma2}
The following function
\begin{align}\label{auxiliary}
&Z(\textbf{V}, \textbf{V}^{(\tau)}) = \alpha^{\lambda} \sum_{rs}\frac{{[\textbf{V}^{(\tau)}\textbf{R}]}_{rs}\textbf{V}_{rs}^2}{\textbf{V}_{rs}^{(\tau)}}+ \frac{\eta}{2} \sum_{rs}\frac{[\textbf{H}\textbf{V}^{(\tau)}  ]_{rs}\textbf{V}_{rs}^2}{\textbf{V}_{rs}^{(\tau)}} -2\alpha^{\lambda} \sum_{rs}\textbf{Q}_{rs}\textbf{V}_{rs}^{(\tau)}(1+log\frac{\textbf{V}_{rs}}{\textbf{V}_{rs}^{(\tau)}})
\end{align}
is an auxiliary function of $J(\textbf{V})$ in problem~(\ref{Probdeduction}).  Furthermore, it is a convex function in $\textbf{V}$ and its global minima is
\begin{align}\label{minima}
\textbf{V}_{rs} = \mathop{\arg\min}_{\textbf{V}}Z(\textbf{V}, \textbf{V}^{(\tau)}) = \textbf{V}_{rs}^{(\tau)}\sqrt{\frac{2\alpha^{\lambda}\textbf{Q}_{rs}}{[2\alpha^{\lambda}\textbf{V}^{(\tau)}\textbf{R}+\eta\textbf{H}\textbf{V}^{(\tau)}]_{rs}}}
\end{align}
\end{lemma}

\begin{proof}
According to the Propositions 3 and 4 in~\cite{ding2008convex}, we have
\begin{align}\label{bound1}
Tr(\textbf{V}^{T}\textbf{V}\textbf{R}) \le  \sum_{rs}\frac{{[\textbf{V}^{(\tau)}\textbf{R}]}_{rs}\textbf{V}_{rs}^2}{\textbf{V}_{rs}^{(\tau)}}
\end{align}

\begin{align}\label{bound2}
Tr(\textbf{V}^{T}\textbf{Q}) \ge \sum_{rs}\textbf{Q}_{rs}\textbf{V}_{rs}^{(\tau)}(1+log\frac{\textbf{V}_{rs}}{\textbf{V}_{rs}^{(\tau)}})
\end{align}
and
\begin{align}\label{bound3}
Tr(\textbf{V}^{T}\textbf{H}\textbf{V}) \le \sum_{rs}\frac{[\textbf{H}\textbf{V}^{(\tau)}  ]_{rs}\textbf{V}_{rs}^2}{\textbf{V}_{rs}^{(\tau)}}.
\end{align}
In terms of Eqs.~(\ref{bound1}) and (\ref{bound3}), we can easily obtain $J(\textbf{V}) \le Z(\textbf{V}, \textbf{V}^{(\tau)})$ and $J(\textbf{V}) = Z(\textbf{V}, \textbf{V})$.  Thus $Z(\textbf{V}, \textbf{V}^{(\tau)})$ is an auxiliary function for $J(\textbf{V})$.

Taking the derivative of  $Z(\textbf{V}, \textbf{V}^{(\tau)})$ w.r.t. $\textbf{V}_{rs}$, we have
\begin{align}\label{Eq77}
&\frac{\partial Z(\textbf{V}, \textbf{V}^{(\tau)})}{\partial \textbf{V}_{rs}}=2\alpha^{\lambda} \frac{{[\textbf{V}^{(\tau)}\text{R}]}_{rs}\textbf{V}_{rs}}{\textbf{V}_{rs}^{(\tau)}}+\eta \frac{[\textbf{H}\textbf{V}^{(\tau)}  ]_{rs}\textbf{V}_{rs}}{\textbf{V}_{rs}^{(\tau)}} -2\alpha^{\lambda} \textbf{Q}_{rs} \frac{\textbf{V}_{rs}^{(\tau)}}{\textbf{V}_{rs}}=0
\end{align}
The Hessian matrix of $Z(\textbf{V}, \textbf{V}^{(\tau)})$ can be calculated by as follows:
\begin{align}
&\frac{\partial^2Z(\textbf{V}, \textbf{V}^{(\tau)})}{\partial \textbf{V}_{rs} \partial \textbf{V}_{pq}}=\delta_{rp}\delta_{sq}(2\alpha^{\lambda}\frac{[\textbf{V}^{(\tau)}\textbf{R}]_{rs}}{\textbf{V}_{rs}^{(\tau)}}+
\eta\frac{[\textbf{H}\textbf{V}^{(\tau)}]_{rs}}{\textbf{V}_{rs}^{(\tau)}} +2\alpha^{\lambda}\textbf{Q}_{rs}\frac{\textbf{V}_{rs}^{(\tau)}}{V_{rs}^{2}}),
\end{align}
where $\delta_{rp}$ equals to 1 if $r=p$, 0 otherwise.  This Hessian matrix is a diagonal matrix with positive elements. Hence, $Z(\textbf{V}, \textbf{V}^{\tau})$ is a convex function of $\textbf{V}$.

By solving Eq.~(\ref{Eq77}), we have the global minima as $\textbf{V}_{rs} = \mathop{\arg\min}_{\textbf{V}}Z(\textbf{V}, \textbf{V}^{(\tau)}) = \textbf{V}_{rs}^{(\tau)}\sqrt{\frac{2\alpha^{\lambda}\textbf{Q}_{rs}}{[2\alpha^{\lambda}\textbf{V}^{(\tau)}\textbf{R}+\eta\textbf{H}\textbf{V}^{(\tau)}]_{rs}}}$.
\end{proof}

\begin{theorem}
Algorithm 1 monotonically decreases the objective function value in problem~(\ref{NewEq1}) by iteratively updating $\textbf{V}$ until convergence.
\end{theorem}

\begin{proof}
Replacing $Z(\textbf{V}, \textbf{V}^{(\tau)})$ in Eq.~(\ref{update}) by Eq.~(\ref{auxiliary}), then we have
\begin{align}
\textbf{V}_{rs}^{(\tau+1)} =  \textbf{V}_{rs}^{(\tau)}\sqrt{\frac{2\alpha^{\lambda}\textbf{Q}_{rs}}{[2\alpha^{\lambda}\textbf{V}^{(\tau)}\textbf{R}+\eta\textbf{H}\textbf{V}^{(\tau)}]_{rs}}}.
\end{align}
According to Lemma~\ref{Lemma2}, we can obtain that $Z(\textbf{V}, \textbf{V}^{(\tau)})$ is the auxiliary function of $J(\textbf{V})$. Then, we have $J(\textbf{V}^{(\tau+1)}) \le Z(\textbf{V}^{\tau+1}, \textbf{V}^{\tau})$ in terms of Definition~\ref{Def1}. Based on Lemma~\ref{Lemma1}, we have $Z(\textbf{V}^{\tau+1}, \textbf{V}^{\tau})\le Z(\textbf{V}^{\tau}, \textbf{V}^{\tau})$. Hence, we conclude that
\begin{align}
J(\textbf{V}^{(\tau+1)}) \le Z(\textbf{V}^{\tau+1}, \textbf{V}^{\tau}) \le Z(\textbf{V}^{\tau}, \textbf{V}^{\tau})=J(\textbf{V}^{(\tau)})
\end{align}
Hence, the objective function value of problem~(\ref{NewEq1}) is monotonically decrease by Algorithm 1 in each iteration.
\end{proof}

\subsection{Complexity Analysis}
In this section, we analyze the computational time complexity of the proposed algorithm I$^2$MUFS. In Algorithm 1, the optimization of model.~(\ref{Eq8}) consists of four parts. For the $t$-th chunk, while updating $\textbf{V}^{(v)}$, the time complexity is $\mathcal{O}(d_{v}k+N_{t}^{2}k)$. While updating $\textbf{U}_{t}^{(v)}$, the computational complexity depends on computing $\textbf{G}_{t}^{(v)} $ and $\textbf{P}_{t}^{(v)} $, whose time complexity are $\mathcal{O}(d_{v}k)$ and $\mathcal{O}(N_{t}^{2}k)$, respectively. Hence, the total time complexity of updating $\textbf{U}_{t}^{(v)}$ is $\mathcal{O}(d_{v}k+N_{t}^{2}k)$. While updating $\textbf{U}_t^{*}$, its time complexity is $\mathcal{O}(N_{t}kn_{v})$. For updating $\alpha^{(v)}$, its computational cost depends on computing $L_{T}^{(v)}$, whose time complexity is $\mathcal{O}(d_{v}N_{t}k)$. Hence, the total time complexity of Algorithm 1 is $\mathcal{O}(\sum\limits_{v=1}^{{{n}_{v}}}{{({d}_{v}k+N_{t}^{2}k+d_{v}N_{t}k)+N_{t}kn_{v}}})$ for each iteration.

\section{Experiments}
In this section, we conduct a series of experiments on some real-world datasets to demonstrate the effectiveness and superiority of the proposed method in comparison with several state-of-the-art methods.

\subsection{Datasets}
In our experiments, eight public available multi-view datasets are used including \textbf{3Sources}\footnote{http://mlg.ucd.ie/datasets/3sources.html}, \textbf{BBCSport}\footnote{http://mlg.ucd.ie/datasets/segment.html \label{BBC}}, \textbf{WebWashington}\footnote{http://membres-lig.imag.fr/grimal/data.html \label{Web}}, \textbf{BBCNews}\textsuperscript{\ref{BBC}}, \textbf{Caltech101-7}\footnote{http://www.vision.caltech.edu/Image-Datasets/Caltech101 \label{Caltech101}}, \textbf{UCIDigits}\footnote{http://archive.ics.uci.edu/ml/datasets/Multiple+Features \label{UCIDigits}}, \textbf{ALOI}\footnote{http://elki.dbs.ifi.lmu.de/wiki/DataSets/MultiView \label{ALOI}} and \textbf{USPS}\footnote{https://www.kaggle.com/bistaumanga/usps-dataset \label{USPS}}.

\textbf{3Sources}  is a multi-view text datasets including 169 news objects collected from three online news medias BBC, Reuters and The Guardian.

\textbf{BBCSport} is a well-known multi-view document dataset, which contains 737 sports news articles from the BBC Sport website in 2004-2005. According to the experimental settings in~\cite{Wen2020Incomplete}, we selected 116 news objects with four views from the whole datasets.

\textbf{WebWashington} contains 203 web-pages derived from computer science department web sites at university of Washington, where each web-page is characterized by three views including the content on the page, the anchor text on the hyper-link and the text in its title.

\textbf{BBCNews} is a widely used multi-view text dataset collected from the BBC news website. It comprises of 685 news articles in five topical areas including business, entertainment, politics, sports and technology. Each articles is split into four segments, where each segment is treated as a view.

\textbf{Caltech101-7} is an objective recognition dataset consisting of 1474 images classified into seven categories. Each image is described by six visual features containing Gabor feature, HOG feature, LBP feature, GIST feature, CENTRIST feature and wavelet moments.

\textbf{UCIDigits} contains 2000 handwritten numerals from 0 to 9 digits classes. Follow~\cite{Wu2019Essential}, three different views' features, namely, morphological features, Fourier coefficients  and pixel averages, are extracted to describe these digit images.

\textbf{ALOI} is a collection of 11025 images, which are characterized by four heterogeneous feature sets, i.e., HSB color histograms, color similarity, RGB color histograms and haralick features.

\textbf{USPS} consists of 9298 digital images divided into 10 classes, where the size of each image is $16\times16$. Two different feature sets, namely, LBP features with 256 dimensions and Gabor Texture features with 32 dimensions are utilized in our experiment.

The detail description of each multi-view dataset is summarized in Table~\ref{Table1}. Moreover, to simulate the incomplete-view setting, we randomly removed certain ratios of instances from the original datasets with complete views. The incomplete sample ratio is set from 10\% to 50\% with a step of 10\%.

\begin{table}[!htp]
\tabcolsep 0pt
\caption{A detail description of datasets}\small \label{Table1}
\vspace*{-15pt}
    \begin{flushleft}
    \def\temptablewidth{\textwidth}
        {\rule{\temptablewidth}{1pt}}
        \begin{tabular*}{\temptablewidth}{@{\extracolsep{\fill}}llcclc}
Datasets & Abbr. &\# Views & \# Instance & \# Features & \# Class\\
\midrule[1pt]
3Sources & 3S&3 & 169 & 3560 / 3631 / 3068 & 6\\
BBCSport &BBCS &4 & 116&  1991 / 2063 / 2113 / 2158 & 5 \\
WebWashington&WebW& 3& 203 &1703 / 230 / 230 &4\\
BBCNews & BBCN&4 & 685 &4659 / 4633 / 4665 / 4684& 5\\
Caltech101-7 & Caltech101-7 &6&1474& 48 / 40 / 254 / 1984 / 512 / 928 & 7\\
UCIDigits& UCIDigits& 3& 2000 & 240 / 76 / 6 & 10\\
ALOI& ALOI& 4& 11025 & 77 / 13 / 64 / 64 & 100\\
USPS& USPS& 2& 9298 & 256 / 32 & 10\\
        \end{tabular*}
        {\rule{\temptablewidth}{1pt}}
        \end{flushleft}
\end{table}

\subsection{Evaluation Metrics}
In order to evaluate the performance of compared methods, we use three popular evaluation metrics, i.e., Normalized Mutual Information (NMI), Adjusted Rand Index (ARI) and F-Measure.

Given $n$ samples, let $\mathcal{C}=\{\mathcal{C}_{1}, \mathcal{C}_{2}, \cdots, \mathcal{C}_{c}\}$ denote the clustering result and $\mathcal{T}=\{\mathcal{T}_{1}, \mathcal{T}_{2}, \cdots, \mathcal{T}_{k}\}$ denote the ground-truth partition. Then, $n_{ij}=|\mathcal{C}_{i}\bigcap \mathcal{T}_{j}|$ indicates the number of common samples both in cluster $\mathcal{C}_{i}$  and ground-truth partition $\mathcal{T}_{j}$, $n_{i}=|\mathcal{C}_{i}|$ indicates the number of samples in cluster $\mathcal{C}_{i}$, and  $m_{j}=|\mathcal{T}_{j}|$ indicates the number of samples in ground-truth partition $\mathcal{T}_{j}$.

NMI measures the consistency between the ground-truth partition $\mathcal{T}$ and the clustering result $\mathcal{C}$ by computing their mutual information. It is defined as follows:
 \begin{equation}\label{NMI}
\text{NMI}=\frac{\sum_{i=1}^{c}\sum_{j=1}^{k}n_{ij}log(\frac{nn_{ij}}{n_{i}m_{j}})}{\sqrt{(\sum_{i=1}^{c}n_{i}log(\frac{n_{i}}{n}))(\sum_{j=1}^{k}m_{j}log(\frac{m_{j}}{n}))}}.
 \end{equation}

ARI measures the similarity between the ground-truth partition and clustering assignment, which is formulated as follows:

 \begin{equation}\label{ARI}
\text{ARI}=\frac{\sum_{i=1}^{c}\sum_{j=1}^{k}\binom{n_{ij}}{2}-\left[\sum_{i=1}^{c}\binom{n_{i}}{2}\sum_{j=1}^{k}\binom{m_{j}}{2}\right]/\binom{n}{2}}{\frac{1}{2}\left[\sum_{i=1}^{c}\binom{n_{i}}{2}+\sum_{j=1}^{k}\binom{m_{j}}{2}\right]-\left[\sum_{i=1}^{c}\binom{n_{i}}{2}\sum_{j=1}^{k}\binom{m_{j}}{2}\right]/\binom{n}{2}}.
 \end{equation}

F-Measure evaluates the clustering performance from the perspective of set matching, which is computed by the harmonic mean of precision and recall. The precision $\mathcal{P}_{i}$  and recall $\mathcal{R}_{i}$ of cluster $\mathcal{C}_{i}$ are respectively defined as follows:
 \begin{align}
   &\mathcal{P}_{i}=\frac{n_{j_{i}}}{n_{i}}\\
   &\mathcal{R}_{i}=\frac{n_{j_{i}}}{m_{j_{i}}},
 \end{align}
where $j_{i}=\underset{j=1}{\overset{k}{max}}\{n_{ij}\}$ is the $j_{i}$-th partition that contain the maximum number of samples from cluster $\mathcal{C}_{i}$. Then, F-Measure is computed as follows:
 \begin{equation}\label{F-Measure}
\text{F-Measure}=\frac{1}{c}\sum_{i=1}^{c}\frac{2\mathcal{P}_{i}\cdot\mathcal{R}_{i}}{\mathcal{P}_{i}+\mathcal{R}_{i}}.
 \end{equation}

\subsection{Experimental Setup}
To validate the effectiveness of the proposed method, we compare the proposed I$^2$MUFS with several state-of-the-art single view and multi-view unsupervised feature selection methods, which are briefly introduced as follows:

$\bullet$ \textbf{Allfeatures}: All original features are employed to be compared with.

$\bullet$ \textbf{LapScore}~\cite{He:Laplacian}: Laplacian score (LapScore) is a representation single-view unsupervised feature selection method based on the filter model. It employs the Laplacian score to select feature for  preserving the local manifold structure of data.

$\bullet$ \textbf{RNE}~\cite{LiuRNE}:  Robust neighborhood embedding (RNE) is a recent  single-view unsupervised feature selection method according to the embedding approach. It utilizes the locally linear embedding algorithm to obtain the feature weight matrix and imposes $\ell_{1}$-norm on the loss function for the selection of features.

$\bullet$ \textbf{OMVFS}~\cite{OMVFS}: Online Unsupervised Multi-view Feature Selection (OMVFS) is an online multi-view unsupervised feature selection method, which integrates the non-negative matrix factorization and graph regularization to incrementally select features.

$\bullet$ \textbf{ASVW}~\cite{Hou:Multi}: Multi-view Unsupervised Feature Selection with Adaptive Similarity and View Weight (ASVW) selects features by adaptively learning a common similarity matrix of different views with a sparse constraint.

$\bullet$ \textbf{CGMV-UFS}~\cite{CGMV-UFS}: Consensus Learning Guided Multi-view Unsupervised Feature Selection (CGMV-UFS) selects features by simultaneously learning the latent feature matrices of different views and the consensus cluster indicator matrix under the framework of NMF.

$\bullet$ \textbf{ASE-UFS}~\cite{ASE}: Adaptive Similarity Embedding for Unsupervised Multi-View Feature Selection (ASE-UMFS) simultaneously learns the global and local structures across different views with the spare $\ell_{1}$-norm constraint for the feature selection of multi-view data.

$\bullet$ \textbf{CoUFC}~\cite{CoUFC}: Co-Learning Non-Negative Correlated and Uncorrelated Features for Multi-View Data (CoUFC) employs the multi-NMF equipped with the graph regularization and sparse regularizer to select the view-specific features and the common features among different views simultaneously.

In order to fairly compare with other baseline methods, we use the same grid search strategy to tune these parameters. The ranges of all the parameters of compared methods are set according to the original literature for the guarantee of obtaining the best results.  In our method, we set the ranges of the parameters $\beta^{(v)}$, $\eta^{(v)}$ and $\theta^{(v)}$ as \{$10^{-2}, 10^{-1}, ... , 10^{2}$\} and $\lambda$ as \{${2, 3, ... , 11}$\}. Besides, $\beta^{(v)}$, $\eta^{(v)}$ and $\theta^{(v)}$ are simply set to be equal for different views, respectively.  Since it is hard to determine the optimal number of selected features in each dataset, we set the feature selection ratio (i.e., the percentage of selected features) varying from 10\% to 50\% with a step of 10\%. Then, we perform the incomplete multi-view clustering method GIMC\_FLSD~\cite{GIMC_FLSD} on the selected features for the evaluation of the performance of theses multi-view unsupervised feature selection methods in terms of NMI, ARI and F-Measure. For all methods, each experiment is repeated 30 times and the average results are reported.

All concerned experiments are performed with  Matlab 2020b on 64-bits Windows 10 operating system with Intel Core i7-8650U CPU 2.11GHZ, 32.0 GB of memory.

\begin{table*}[!htbp]
\centering
\caption{NMI values of different methods on eight datasets with different ratios of selected features.}\label{Table2}
\vspace{-0.2cm}
\resizebox{\textwidth}{!}{
\renewcommand\tabcolsep{5 pt}
\begin{tabular}{ccccccccccc}
\toprule[1pt]
Datasets & RFS &I$^2$MUFS &Allfeatures&LapScore&RNE&OMVFS&ASVW&CGMV-UFS & ASE-UFS&CoUFC \\
\midrule[0.5pt]
\multirow{5}*{3S}
~& 0.1 &\textbf{0.5974}  & 0.5679$\bullet$ &0.4380$\bullet$& 0.4901$\bullet$& 0.5319$\bullet$ & 0.5360$\bullet$& 0.2857$\bullet$  & 0.5494$\bullet$  & 0.4645$\bullet$ \\
~& 0.2&\textbf{0.6127} & 0.5679$\bullet$  &0.5661$\bullet$ &0.5320$\bullet$ & 0.5457$\bullet$& 0.5513$\bullet$  & 0.2439$\bullet$  & 0.5975& 0.4092$\bullet$  \\
~& 0.3&\textbf{0.6066} & 0.5679$\bullet$ &0.5405$\bullet$&0.5647$\bullet$ & 0.6025  & 0.5611$\bullet$  & 0.2211$\bullet$ & 0.5954 & 0.3608$\bullet$ \\
~& 0.4&\textbf{0.6161} & 0.5679$\bullet$  &0.5663$\bullet$&0.5963  & 0.6018  & 0.5772$\bullet$  & 0.2009$\bullet$  & 0.5966  & 0.3696$\bullet$  \\
~& 0.5&\textbf{0.6114}  & 0.5679$\bullet$&0.5766$\bullet$&0.5972 & 0.6017  & 0.6065 & 0.2191$\bullet$  & 0.6081  & 0.3764$\bullet$\\
\bottomrule[1pt]

\multirow{5}*{BBCS}
~&0.1 & 0.6314           & \textbf{0.6402}  & 0.1821$\bullet$  & 0.2080$\bullet$  & 0.4559$\bullet$  & 0.6307  & 0.2858$\bullet$  & 0.5931$\bullet$                                                  & 0.5557$\bullet$  \\
~&0.2 & \textbf{0.6667}  & 0.6402$\bullet$           & 0.1564$\bullet$  & 0.3311$\bullet$  & 0.4363$\bullet$  & 0.6154$\bullet$ & 0.2299$\bullet$  & 0.6628 & 0.4427$\bullet$ \\
~&0.3 & \textbf{0.6840} & 0.6402$\bullet$          & 0.1861$\bullet$  & 0.3360$\bullet$  & 0.4889$\bullet$  & 0.5734$\bullet$                                                  & 0.2283$\bullet$  & 0.6744  & 0.4210$\bullet$  \\
~&0.4 & \textbf{0.6842}          & 0.6402$\bullet$           & 0.2043$\bullet$  & 0.4410$\bullet$  & 0.5426$\bullet$  & 0.5884$\bullet$                         & 0.2027$\bullet$  & 0.6629$\bullet$                          & 0.4348$\bullet$  \\
~&0.5 & \textbf{0.6749}  & 0.6402$\bullet$           & 0.2436$\bullet$  & 0.5394 $\bullet$& 0.5633$\bullet$  & 0.5806$\bullet$                        & 0.1813$\bullet$ & 0.6457$\bullet$                          & 0.3778$\bullet$ \\
\bottomrule[1pt]

\multirow{5}*{WebW}
~&0.1 & \textbf{0.3008} & 0.2832 & 0.1821$\bullet$  & 0.1175$\bullet$ & 0.0701$\bullet$  & 0.2248$\bullet$ & 0.1363$\bullet$   & 0.1837$\bullet$      & 0.0980$\bullet$  \\
~&0.2 & \textbf{0.3213}  & 0.2832$\bullet$                                  & 0.1564$\bullet$ & 0.1742$\bullet$ & 0.1404$\bullet$  & 0.2566$\bullet$                          & 0.1995$\bullet$                         & 0.1853$\bullet$  & 0.1388$\bullet$  \\
~&0.3 & \textbf{0.3502}  & 0.2832$\bullet$                                  & 0.1861$\bullet$ & 0.2596$\bullet$  & 0.1736$\bullet$  & 0.2910$\bullet$                          & 0.2957$\bullet$                         & 0.2072$\bullet$ & 0.1850$\bullet$ \\
~&0.4 & \textbf{0.3582}  & 0.2832$\bullet$                                & 0.2043$\bullet$& 0.1709$\bullet$ & 0.2161$\bullet$  & 0.3054$\bullet$                          & 0.2991$\bullet$                          & 0.2123$\bullet$                           & 0.1457$\bullet$  \\
~&0.5 & \textbf{0.3468}  & 0.2832$\bullet$                                 & 0.2436$\bullet$ & 0.2120$\bullet$  & 0.2636$\bullet$  & 0.3138$\bullet$                          & 0.3270 & 0.2123$\bullet$                           & 0.1392$\bullet$ \\
\bottomrule[1pt]

\multirow{5}*{BBCN}
~&0.1 & 0.5449                  & \textbf{0.5825}  & 0.3089$\bullet$  & 0.2865$\bullet$ & 0.4004$\bullet$  & 0.4958$\bullet$  & 0.5250 & 0.3392$\bullet$ & 0.3565$\bullet$  \\
~&0.2 & 0.5794           & \textbf{0.5825}          & 0.3775$\bullet$ & 0.5419$\bullet$  & 0.5238$\bullet$  & 0.5108$\bullet$  & 0.5870  & 0.3550$\bullet$ & 0.3804$\bullet$  \\
~&0.3 & \textbf{0.6046}  & 0.5825$\bullet$         & 0.3946$\bullet$  & 0.5024$\bullet$ & 0.5227$\bullet$  & 0.5343$\bullet$  & 0.5717$\bullet$                          & 0.3666$\bullet$  & 0.3262$\bullet$ \\
~&0.4 & \textbf{0.6277}  & 0.5825$\bullet$          & 0.3708$\bullet$ & 0.5089$\bullet$  & 0.5023$\bullet$ & 0.5187$\bullet$  & 0.5715$\bullet$                        & 0.3656$\bullet$  & 0.3273$\bullet$  \\
~&0.5 & \textbf{0.6313}  & 0.5825$\bullet$           & 0.3735$\bullet$  & 0.5015$\bullet$  & 0.4741$\bullet$  & 0.5335$\bullet$  & 0.5678$\bullet$                          & 0.3652$\bullet$  & 0.2787$\bullet$ \\
\bottomrule[1pt]

\multirow{5}*{Caltech101-7}
~&0.1 & \textbf{0.5188} & 0.4280$\bullet$  & 0.4106$\bullet$  & 0.1812$\bullet$ & 0.3976$\bullet$                        & 0.3454$\bullet$  & 0.3973$\bullet$  & 0.2329$\bullet$ & 0.3700$\bullet$  \\
~&0.2 & \textbf{0.5183}  & 0.4280$\bullet$  & 0.3865$\bullet$  & 0.2655$\bullet$ & 0.4046$\bullet$ & 0.3185$\bullet$ & 0.3931$\bullet$  & 0.2428$\bullet$  & 0.3851$\bullet$  \\
~&0.3 & \textbf{0.5004}  & 0.4280$\bullet$  & 0.4364$\bullet$  & 0.2768$\bullet$  & 0.3945$\bullet$  & 0.3189$\bullet$  & 0.3917$\bullet$ & 0.2456$\bullet$  & 0.3640$\bullet$ \\
~&0.4 & \textbf{0.4893}  & 0.4280$\bullet$  & 0.4640$\bullet$  & 0.2807$\bullet$  & 0.3751$\bullet$                         & 0.3101$\bullet$ & 0.4142$\bullet$  & 0.2431$\bullet$ & 0.3644$\bullet$ \\
~&0.5 & \textbf{0.4814}  & 0.4280$\bullet$  & 0.4583$\bullet$  & 0.2885$\bullet$  & 0.3712$\bullet$ & 0.3426$\bullet$  & 0.4103$\bullet$  & 0.2399$\bullet$ & 0.3685$\bullet$ \\
\bottomrule[1pt]

\multirow{5}*{UCIDigits}
~&0.1 & 0.3955                  & \textbf{0.4396}      & 0.3401$\bullet$ & 0.3433$\bullet$  & 0.3355$\bullet$        & 0.0746$\bullet$  & 0.3419$\bullet$  & 0.3709$\bullet$     & 0.2895$\bullet$ \\
~&0.2 & 0.4387          & \textbf{0.4396}    & 0.3646$\bullet$ & 0.3816$\bullet$  & 0.3964$\bullet$      & 0.2150$\bullet$  & 0.3820$\bullet$ &  0.4230  & 0.3000$\bullet$  \\
~&0.3 & \textbf{0.4577}  &  0.4396  & 0.3892$\bullet$ & 0.3868$\bullet$ & 0.4192$\bullet$   & 0.3192$\bullet$  & 0.3736$\bullet$ & 0.4353$\bullet$   & 0.3253$\bullet$ \\
~&0.4 & \textbf{0.4536}  &  0.4396  & 0.4300$\bullet$  & 0.3746$\bullet$  & 0.4279$\bullet$   & 0.3465$\bullet$ & 0.3888$\bullet$  &  0.4343  & 0.3236$\bullet$ \\
~&0.5 & \textbf{0.4550}  &  0.4396     & 0.4128$\bullet$  & 0.3746$\bullet$  &  0.4357  & 0.3885$\bullet$  & 0.3765$\bullet$  & 0.4336$\bullet$& 0.3196$\bullet$ \\
\bottomrule[1pt]

\multirow{5}*{ALOI}
~&0.1 & \textbf{0.5330}  & 0.4142$\bullet$ & 0.4715$\bullet$ & 0.4840$\bullet$ &  0.4961$\bullet$   & 0.4332$\bullet$  & 0.4803$\bullet$ & 0.4452$\bullet$   & 0.0500$\bullet$ \\
~&0.2 & \textbf{0.5303} & 0.4142$\bullet$  & 0.4947$\bullet$  & 0.4488$\bullet$  &  0.5228  & 0.4595$\bullet$  & 0.4455$\bullet$  &0.4206$\bullet$   & 0.1419$\bullet$  \\
~&0.3 & \textbf{0.5468}  & 0.4142$\bullet$  & 0.5051$\bullet$ & 0.4761$\bullet$     &  0.5436    & 0.4584$\bullet$  & 0.4690$\bullet$  & 0.4118$\bullet$   & 0.3876$\bullet$ \\
~&0.4 & \textbf{0.5497}  & 0.4142$\bullet$  & 0.5130$\bullet$  & 0.5091$\bullet$    & 0.5156$\bullet$  & 0.4474$\bullet$  & 0.5051$\bullet$ & 0.4181$\bullet$  & 0.1813$\bullet$  \\
~&0.5 & \textbf{0.5467}  & 0.4142$\bullet$  & 0.5215$\bullet$  & 0.5170$\bullet$   & 0.5189$\bullet$   & 0.4312$\bullet$  & 0.5162$\bullet$ & 0.3943$\bullet$   & 0.3111$\bullet$  \\
\bottomrule[1pt]

\multirow{5}*{USPS}
~&0.1 & \textbf{0.1555}  &  0.1356 & 0.0801$\bullet$   & 0.0438$\bullet$   & 0.0870$\bullet$   & 0.1292$\bullet$    & 0.0434$\bullet$  & 0.1109$\bullet$  & 0.0121$\bullet$ \\
~&0.2 & \textbf{0.1839}  & 0.1356$\bullet$  & 0.1304$\bullet$  & 0.0952$\bullet$  & 0.1658  & 0.1556$\bullet$    & 0.0953$\bullet$  & 0.1305$\bullet$   & 0.0188$\bullet$  \\
~&0.3 & \textbf{0.1933}  & 0.1356$\bullet$ & 0.1418$\bullet$ & 0.1343$\bullet$    &  0.1847  & 0.1620$\bullet$     & 0.1327$\bullet$  & 0.1369$\bullet$  & 0.0147$\bullet$ \\
~&0.4 & \textbf{0.1997}  & 0.1356$\bullet$  & 0.1415$\bullet$  & 0.1529$\bullet$  &  0.1803  &  0.1836  & 0.1556$\bullet$  & 0.1295$\bullet$   & 0.0155$\bullet$ \\
~&0.5 & \textbf{0.2190}  & 0.1356$\bullet$  & 0.1517$\bullet$ & 0.1775$\bullet$  & 0.1799$\bullet$ &  0.2002  & 0.1788$\bullet$  & 0.1332$\bullet$    & 0.0134$\bullet$ \\
\bottomrule[1pt]
\end{tabular}
}
\end{table*}


\begin{table*}[!htbp]
\centering
\caption{ARI values of different methods on eight datasets with different ratios of selected features.}\label{Table3}
\vspace{-0.2cm}
\resizebox{\textwidth}{!}{
\renewcommand\tabcolsep{5 pt}
\begin{tabular}{ccccccccccc}
\toprule[1pt]
Datasets & RFS &I$^2$MUFS &Allfeatures&LapScore&RNE&OMVFS&ASVW&CGMV-UFS & ASE-UFS&CoUFC \\
\midrule[0.5pt]
\multirow{5}*{3S}
~& 0.1  &\textbf{0.5135}  & 0.4246$\bullet$&0.3588$\bullet$& 0.3721$\bullet$ & 0.4442$\bullet$ & 0.4231$\bullet$  & 0.0488$\bullet$  & 0.4430$\bullet$ & 0.3921$\bullet$ \\
~& 0.2 & \textbf{0.5372} & 0.4246$\bullet$ &0.4971$\bullet$& 0.3918$\bullet$ & 0.4297$\bullet$& 0.4233$\bullet$  & 0.0111$\bullet$ & 0.5126$\bullet$& 0.3016$\bullet$ \\
~& 0.3& \textbf{0.5539} & 0.4246$\bullet$& 0.4340$\bullet$& 0.4548$\bullet$  & 0.5210$\bullet$  & 0.4551$\bullet$  & 0.0085$\bullet$& 0.4829$\bullet$ & 0.2807$\bullet$ \\
~& 0.4& \textbf{0.5439}  & 0.4246$\bullet$&0.5069$\bullet$ &0.5104$\bullet$  & 0.5229$\bullet$  & 0.4698$\bullet$  & 0.0033$\bullet$ & 0.4974$\bullet$  & 0.3458$\bullet$ \\
~& 0.5& \textbf{0.5692} & 0.4246$\bullet$ & 0.5080$\bullet$&0.5242$\bullet$  & 0.5226$\bullet$ & 0.4991$\bullet$  & 0.0186$\bullet$  & 0.5212$\bullet$ & 0.2778$\bullet$\\
\bottomrule[1pt]

\multirow{5}*{BBCS}
~&0.1    & \textbf{0.6034}  & 0.5735$\bullet$ & 0.1187$\bullet$  & 0.1110$\bullet$  & 0.3876$\bullet$  & 0.5967  & 0.1361$\bullet$  & 0.5216$\bullet$                         & 0.4537$\bullet$ \\
~&0.2 & \textbf{0.6417}       & 0.5735$\bullet$ & 0.0945$\bullet$  & 0.2503$\bullet$ & 0.3566$\bullet$ & 0.6021$\bullet$  & 0.0848$\bullet$  & 0.5673$\bullet$ & 0.3450$\bullet$ \\
~&0.3 & \textbf{0.6574}           & 0.5735$\bullet$  & 0.0925$\bullet$  & 0.2212$\bullet$  & 0.4222$\bullet$                         & 0.5460$\bullet$                          & 0.0791$\bullet$  & 0.6427  & 0.3648$\bullet$ \\
~&0.4  & \textbf{0.6418}               & 0.5735$\bullet$  & 0.1435$\bullet$  & 0.3455$\bullet$ & 0.4455$\bullet$  & 0.5644$\bullet$                          & 0.0652$\bullet$ & 0.5998$\bullet$                          & 0.3584$\bullet$ \\
~&0.5  & \textbf{0.6459}          & 0.5735$\bullet$  & 0.1683$\bullet$ & 0.4831$\bullet$  & 0.5026$\bullet$                         & 0.5453$\bullet$                        & 0.0444$\bullet$   & 0.5490$\bullet$                          & 0.3384$\bullet$\\
\bottomrule[1pt]

\multirow{5}*{WebW}
~&0.1  & \textbf{0.3699}  & 0.3579  & 0.1187$\bullet$  & 0.1503$\bullet$  & 0.1033$\bullet$  & 0.2802$\bullet$   & 0.1732$\bullet$      & 0.2195$\bullet$                          & 0.0537$\bullet$ \\
~&0.2 & \textbf{0.3896}                                  & 0.3579$\bullet$                         & 0.0945$\bullet$  & 0.2094$\bullet$ & 0.1860$\bullet$                        & 0.2890$\bullet$                         & 0.2191$\bullet$  & 0.2232$\bullet$                          & 0.1515$\bullet$ \\
~&0.3  & \textbf{0.4237}                                & 0.3579$\bullet$                         & 0.0925$\bullet$  & 0.3002$\bullet$ & 0.1853$\bullet$                        & 0.3482$\bullet$                         & 0.3442$\bullet$  & 0.2478$\bullet$  & 0.2823$\bullet$ \\
~&0.4 & \textbf{0.4235}                                  & 0.3579$\bullet$                          & 0.1435$\bullet$& 0.2117$\bullet$ & 0.2627$\bullet$                         & 0.3441$\bullet$                           & 0.3926$\bullet$  & 0.2585$\bullet$                         & 0.1814$\bullet$ \\
~&0.5 & \textbf{0.4256}                                  & 0.3579$\bullet$                         & 0.1683$\bullet$  & 0.2269$\bullet$ & 0.3257$\bullet$                       & 0.3736$\bullet$                         & 0.3842$\bullet$ & 0.2585$\bullet$                          & 0.1840$\bullet$\\
\bottomrule[1pt]

\multirow{5}*{BBCN}
~&0.1     & 0.5171         & \textbf{0.5345} & 0.1943$\bullet$  & 0.2474$\bullet$ & 0.3903$\bullet$ & 0.4375$\bullet$ & 0.4487$\bullet$  & 0.1495$\bullet$ & 0.3162$\bullet$ \\
~&0.2       & \textbf{0.5598} & 0.5345$\bullet$       & 0.2707$\bullet$  & 0.5157$\bullet$ & 0.4799$\bullet$  & 0.4894$\bullet$ & 0.5229$\bullet$  & 0.1566$\bullet$ & 0.3634$\bullet$ \\
~&0.3  & \textbf{0.5910}          & 0.5345$\bullet$         & 0.3147$\bullet$ & 0.4336$\bullet$  & 0.4424$\bullet$  & 0.5087$\bullet$   & 0.4994$\bullet$ & 0.1645$\bullet$ & 0.2823$\bullet$ \\
~&0.4  & \textbf{0.6324}          & 0.5345$\bullet$           & 0.2904$\bullet$  & 0.4312$\bullet$ & 0.4209$\bullet$  & 0.4946$\bullet$ & 0.5071$\bullet$ & 0.1630$\bullet$ & 0.2736$\bullet$ \\
~&0.5 & \textbf{0.6566}      & 0.5345$\bullet$          & 0.3025$\bullet$  & 0.4425$\bullet$  & 0.3804$\bullet$ & 0.5394$\bullet$  & 0.5057$\bullet$ & 0.1636$\bullet$ & 0.2450$\bullet$\\
\bottomrule[1pt]

\multirow{5}*{Caltech101-7}
~&0.1  & \textbf{0.3805} & 0.2740$\bullet$ & 0.2806$\bullet$                         & 0.1422$\bullet$ & 0.3478$\bullet$                         & 0.2812$\bullet$ & 0.2064$\bullet$  & 0.0575$\bullet$  & 0.2935$\bullet$ \\
~&0.2  & \textbf{0.3590}  & 0.2740$\bullet$ & 0.2292$\bullet$                         & 0.2056$\bullet$ & 0.3510  & 0.2548$\bullet$& 0.2087$\bullet$  & 0.0516$\bullet$ & 0.3026$\bullet$ \\
~&0.3  & \textbf{0.3422} & 0.2740$\bullet$ & 0.2434$\bullet$                          & 0.2232$\bullet$  &  0.3413  & 0.2553$\bullet$  & 0.2212$\bullet$  & 0.0515$\bullet$  & 0.2746$\bullet$ \\
~&0.4 & \textbf{0.3339} & 0.2740$\bullet$ & 0.2975$\bullet$                          & 0.2202$\bullet$ & 0.3072$\bullet$                         & 0.2490$\bullet$ & 0.2313$\bullet$ & 0.0496$\bullet$  & 0.2733$\bullet$ \\
~&0.5  & \textbf{0.3098} & 0.2740$\bullet$  &  0.2902  & 0.2260$\bullet$  &  0.3113  & 0.2830$\bullet$  & 0.2250$\bullet$ & 0.0480$\bullet$ & 0.2794$\bullet$\\
\bottomrule[1pt]

\multirow{5}*{UCIDigits}
~&0.1        & 0.2709             & \textbf{0.2918}    & 0.2462$\bullet$  & 0.2194$\bullet$   & 0.2361$\bullet$   & 0.0274$\bullet$  & 0.2140$\bullet$ & 0.1585$\bullet$ & 0.1807$\bullet$ \\
~&0.2       & \textbf{0.2980}    &  0.2918  & 0.2467$\bullet$  & 0.2419$\bullet$   & 0.2661$\bullet$    & 0.1254$\bullet$  & 0.2424$\bullet$ & 0.1828$\bullet$  & 0.2012$\bullet$ \\
~&0.3 & \textbf{0.3176}  & 0.2918$\bullet$  & 0.2537$\bullet$  & 0.2508$\bullet$   &  0.3098  & 0.2018$\bullet$  & 0.2338$\bullet$  & 0.1980$\bullet$ & 0.2277$\bullet$ \\
~&0.4  & \textbf{0.3116} & 0.2918 & 0.2855$\bullet$  & 0.2323$\bullet$ & 0.2864$\bullet$ & 0.2106$\bullet$  & 0.2447$\bullet$ & 0.1864$\bullet$  & 0.2187$\bullet$ \\
~&0.5 & \textbf{0.3156}  & 0.2918$\bullet$   & 0.2724$\bullet$  & 0.2323$\bullet$&  0.2961 & 0.2528$\bullet$ & 0.2336$\bullet$& 0.1902$\bullet$  & 0.2242$\bullet$\\
\bottomrule[1pt]

\multirow{5}*{ALOI}
~&0.1  & \textbf{0.1629}  & 0.0968$\bullet$ & 0.1357$\bullet$ &  0.1429  &  0.1548  & 0.1079$\bullet$ & 0.1411$\bullet$ &0.1181$\bullet$   & 0.0876$\bullet$ \\
~&0.2  & \textbf{0.1972} & 0.0968$\bullet$  & 0.1721$\bullet$  & 0.1263$\bullet$  & 0.1642$\bullet$  & 0.1243$\bullet$  & 0.1245$\bullet$ & 0.1034$\bullet$   & 0.0948$\bullet$ \\
~&0.3 & \textbf{0.2104}  & 0.0968$\bullet$  & 0.1760$\bullet$  & 0.1343$\bullet$     & 0.1674$\bullet$   & 0.1221$\bullet$  & 0.1311$\bullet$  & 0.0964$\bullet$   & 0.1118$\bullet$ \\
~&0.4  & \textbf{0.2125} & 0.0968$\bullet$  & 0.1763$\bullet$ & 0.1626$\bullet$   & 0.1745$\bullet$ & 0.1210$\bullet$  & 0.1585$\bullet$  &0.1056$\bullet$   & 0.0236$\bullet$ \\
~&0.5  & \textbf{0.2092}  & 0.0968$\bullet$ & 0.1714$\bullet$  & 0.1659$\bullet$    & 0.1759$\bullet$  & 0.1202$\bullet$  & 0.1662$\bullet$   & 0.1015$\bullet$   & 0.0687$\bullet$ \\
\bottomrule[1pt]

\multirow{5}*{USPS}
~&0.1  & \textbf{0.0842} &  0.0828  & 0.0335$\bullet$   & 0.0027$\bullet$   & 0.0177$\bullet$  & 0.0775  & 0.0026$\bullet$  & 0.0632$\bullet$ & 0.0046$\bullet$ \\
~&0.2& \textbf{0.0896} &  0.0828  &  0.0709  & 0.0400$\bullet$  &  0.0768   & 0.0800   & 0.0397$\bullet$    & 0.0820   & 0.0061$\bullet$ \\
~&0.3  & \textbf{0.0935}  &  0.0828  &  0.0742  & 0.0728$\bullet$    &  0.0820    & 0.0909  & 0.0724$\bullet$  & 0.0847   & 0.0072$\bullet$ \\
~&0.4  & \textbf{0.0977}  &  0.0828  & 0.0751$\bullet$  &  0.0908 &  0.0817  & 0.0927   &  0.0905 & 0.0960  & 0.0052$\bullet$\\
~&0.5  & \textbf{0.1185}   & 0.0828$\bullet$  & 0.0810$\bullet$&  0.1041& 0.0819$\bullet$& 0.1054 &  0.1056   &0.0878$\bullet$   & 0.0040$\bullet$ \\
\bottomrule[1pt]
\end{tabular}
}
\end{table*}


\begin{table*}[!htbp]
\centering
\caption{F-Measure values of different methods on eight datasets with different ratios of selected features.}\label{Table4}
\vspace{-0.2cm}
\resizebox{\textwidth}{!}{
\renewcommand\tabcolsep{5 pt}
\begin{tabular}{ccccccccccc}
\toprule[1pt]
Datasets & RFS &I$^2$MUFS &Allfeatures&LapScore&RNE&OMVFS&ASVW&CGMV-UFS & ASE-UFS&CoUFC \\
\midrule[0.5pt]
\multirow{5}*{3S}
~& 0.1  &\textbf{0.7037}    &0.6301$\bullet$    &0.6201 $\bullet$   &0.5887 $\bullet$   &0.6322 $\bullet$   &0.5962 $\bullet$   &0.4280 $\bullet$   &0.6701     $\bullet$&0.5014 $\bullet$ \\
~& 0.2 &\textbf{0.7177}     &0.6301 $\bullet$   &0.6777 $\bullet$   &0.6429 $\bullet$   &0.6483 $\bullet$   &0.6334 $\bullet$   &0.4058 $\bullet$   &0.6843 $\bullet$   &0.5480 $\bullet$ \\
~& 0.3& \textbf{0.7215} &   0.6301$\bullet$&    0.6527 $\bullet$&   0.6864 $\bullet$&   0.7124 &    0.6170 $\bullet$    &0.4122$\bullet$ &  0.6871  $\bullet$&0.5256$\bullet$ \\
~& 0.4& \textbf{0.7277}  &  0.6301 $\bullet$ &  0.6781 $\bullet$     &0.7112     &0.7196  & 0.6556  $\bullet$&  0.3994$\bullet$  &  0.6934  $\bullet$&  0.5166 $\bullet$ \\
~& 0.5& \textbf{0.7389} &   0.6301$\bullet$&    0.6721$\bullet$ &   0.7037 $\bullet$&   0.7191 &    0.6451$\bullet$ &   0.4119 $\bullet$&   0.7064 $\bullet$&   0.5170$\bullet$ \\
\bottomrule[1pt]

\multirow{5}*{BBCS}
~&0.1    & \textbf{0.7991} &    0.7741 $\bullet$  &    0.4702$\bullet$ &   0.4736$\bullet$ &   0.6805 $\bullet$&   0.6482 $\bullet$&   0.5340 $\bullet$&   0.7757 $\bullet$  &    0.5041 $\bullet$ \\
~&0.2 &\textbf{0.8068} &    0.7741$\bullet$ &   0.4271$\bullet$ &   0.5944$\bullet$ &   0.6779 $\bullet$&   0.6260 $\bullet$&   0.4884 $\bullet$&   0.7582 $\bullet$&   0.6556 $\bullet$\\
~&0.3 &\textbf{0.8219}  &0.7741 $\bullet$&  0.4142 $\bullet$&   0.5658 $\bullet$&   0.7167$\bullet$&    0.6745 $\bullet$&   0.4783$\bullet$ &   0.7677 $\bullet$&   0.5565 $\bullet$\\
~&0.4  & \textbf{0.8176}    &0.7741 $\bullet$   &0.5076$\bullet$    &0.6645 $\bullet$   &0.7220$\bullet$    &0.6813 $\bullet$   &0.4587 $\bullet$   &0.7845 $\bullet$   &0.5739 $\bullet$ \\
~&0.5  & \textbf{0.8169}    &0.7741 $\bullet$   &0.4993$\bullet$    &0.7405$\bullet$    &0.7504 $\bullet$   &0.7179 $\bullet$   &0.4424 $\bullet$   &0.7882 $\bullet$   &0.6217 $\bullet$\\
\bottomrule[1pt]

\multirow{5}*{WebW}
~&0.1  & \textbf{0.6970}    &0.6924     &0.4702 $\bullet$   &0.4373 $\bullet$   &0.5327 $\bullet$   &0.6282 $\bullet$   &0.5580 $\bullet$   &0.3402 $\bullet$   &0.5128 $\bullet$\\
~&0.2 & \textbf{0.7027} &   0.6924  & 0.4271$\bullet$ &     0.4613$\bullet$     & 0.5765    $\bullet$& 0.5958   $\bullet$& 0.6102 $\bullet$ & 0.3402    $\bullet$& 0.4895 $\bullet$ \\
~&0.3  & \textbf{0.7206}    &0.6924     $\bullet$&0.4142    $\bullet$&0.4859 $\bullet$  &0.5913     $\bullet$&0.6798 $\bullet$  &0.6804 $\bullet$   &0.3489 $\bullet$   &0.5001 $\bullet$\\
~&0.4 & \textbf{0.7283}     &0.6924 $\bullet$   &0.5076     $\bullet$&0.4767 $\bullet$  &0.6521     $\bullet$&0.6648 $\bullet$  &0.6952 $\bullet$   &0.3450 $\bullet$   &0.4856 $\bullet$ \\
~&0.5 & \textbf{0.7276}     &0.6924 $\bullet$   &0.4993 $\bullet$   &0.4648 $\bullet$   &0.6599 $\bullet$   &0.6782     $\bullet$&0.6939 $\bullet$  &0.3447 $\bullet$   &0.5579$\bullet$ \\
\bottomrule[1pt]

\multirow{5}*{BBCN}
~&0.1     & 0.7412  & \textbf{0.7640} $\bullet$ & 0.5398$\bullet$   & 0.5051 $\bullet$  & 0.6470 $\bullet$  & 0.5690 $\bullet$  & 0.7201 $\bullet$  & 0.3402    $\bullet$& 0.3346 $\bullet$ \\
~&0.2       &\textbf{0.7740}    &0.7640     &0.6164$\bullet$    &0.7248$\bullet$    &0.7123 $\bullet$   &0.5955 $\bullet$   &0.7635     &0.3402 $\bullet$   &0.3181 $\bullet$ \\
~&0.3  & \textbf{0.7872}    &0.7645 &0.6385 $\bullet$   &0.6945 $\bullet$   &0.7022 $\bullet$   &0.6420$\bullet$    &0.7474 $\bullet$   &0.3489 $\bullet$   &0.3384$\bullet$ \\
~&0.4  & \textbf{0.8012}    &0.7640 $\bullet$   &0.5959 $\bullet$   &0.6911 $\bullet$   &0.6851 $\bullet$   &0.6781 $\bullet$   &0.7491 $\bullet$   &0.3450 $\bullet$   &0.2940 $\bullet$\\
~&0.5 & \textbf{0.8096}      & 0.7640   $\bullet$ & 0.6034 $\bullet$ &  0.6945 $\bullet$     & 0.6639   $\bullet$ & 0.6943  $\bullet$ & 0.7487 $\bullet$     & 0.3447 $\bullet$  & 0.3334 $\bullet$\\
\bottomrule[1pt]

\multirow{5}*{Caltech101-7}
~&0.1  & \textbf{0.6282}    &0.5346 $\bullet$   &0.5084 $\bullet$   &0.4423 $\bullet$   &0.6022 $\bullet$   &0.4823     $\bullet$&0.5029 $\bullet$  &0.2394 $\bullet$   &0.5136 $\bullet$ \\
~&0.2  & \textbf{0.6070}    &0.5346 $\bullet$   &0.4880 $\bullet$   &0.4823 $\bullet$   &0.6022     &0.5040 $\bullet$   &0.5007 $\bullet$   &0.2198 $\bullet$   &0.5182 $\bullet$\\
~&0.3  & \textbf{0.5897}     &0.5346 $\bullet$   &0.5232$\bullet$ &  0.4958 $\bullet$&   0.5690$\bullet$     &0.5249 $\bullet$   &0.4980 $\bullet$   &0.2189 $\bullet$   &0.5384 $\bullet$\\
~&0.4 & \textbf{0.5832} &   0.5346 $\bullet$    &0.5585 $\bullet$   &0.4890 $\bullet$&  0.5680$\bullet$ &   0.5171 $\bullet$&   0.5176 $\bullet$    &0.2197 $\bullet$   &0.5281 $\bullet$\\
~&0.5  & \textbf{0.5783}    &0.5346 $\bullet$   &0.5555 $\bullet$   &0.5128 $\bullet$   &0.5678     &0.5028 $\bullet$   &0.5081     $\bullet$&0.2154 $\bullet$  &0.5308$\bullet$ \\
\bottomrule[1pt]

\multirow{5}*{UCIDigits}
~&0.1        & \textbf{0.4796}  &0.4628  &0.4398 $\bullet$   &0.4199 $\bullet$   &0.4311 $\bullet$   &0.1874 $\bullet$   &0.4111 $\bullet$   &0.3361 $\bullet$   &0.1612 $\bullet$ \\
~&0.2       &\textbf{0.4845} &0.4628  $\bullet$ &0.4417 $\bullet$    &0.4406     $\bullet$&0.4642 $\bullet$ &0.2629 $\bullet$   &0.4399 $\bullet$   &0.3475 $\bullet$   &0.3880$\bullet$ \\
~&0.3 & \textbf{0.4992}&    0.4628 $\bullet$    &0.4435$\bullet$    &0.4402 $\bullet$   &0.4734     $\bullet$&0.4112 $\bullet$  &0.4292 $\bullet$   &0.3661 $\bullet$&  0.4290 $\bullet$\\
~&0.4  &\textbf{0.4697} &   0.4628 &    0.4614  &0.4263 $\bullet$   &0.4625     &0.4193 $\bullet$   &0.4400 $\bullet$   &0.3636 $\bullet$   &0.2851 $\bullet$ \\
~&0.5 & \textbf{0.4681}     &0.4628     &0.4553 $\bullet$   &0.4244 $\bullet$   &0.4619     &0.4488 $\bullet$   &0.4307 $\bullet$   &0.3507 $\bullet$   &0.4355 $\bullet$\\
\bottomrule[1pt]

\multirow{5}*{ALOI}
~&0.1  & \textbf{0.4068}    &0.3457     $\bullet$&0.2766 $\bullet$  &0.3267 $\bullet$   &0.3180 $\bullet$   &0.2679 $\bullet$   &0.3219 $\bullet$   &0.2668 $\bullet$   &0.1321 $\bullet$ \\
~&0.2  &\textbf{0.3711}     &0.3457 $\bullet$   &0.3566     &0.2779 $\bullet$   &0.3556 $\bullet$   &0.2975 $\bullet$   &0.2754 $\bullet$   &0.2457 $\bullet$   &0.1702 $\bullet$ \\
~&0.3 & \textbf{0.3794} &    0.3457 $\bullet$&   0.3618     &0.3111 $\bullet$   &0.3807     &0.2952 $\bullet$   &0.3049 $\bullet$   &0.2181 $\bullet$   &0.2100 $\bullet$ \\
~&0.4  & 0.3748     &0.3457 $\bullet$   &0.3915     &0.3475 $\bullet$   &\textbf{0.3941}    &0.2817 $\bullet$   &0.3424 $\bullet$   &0.2587     $\bullet$&0.1623 $\bullet$ \\
~&0.5  & 0.3730     &0.3457     $\bullet$&0.3847$\bullet$   &0.3527 $\bullet$   &\textbf{0.4005}    &0.2587 $\bullet$   &0.3534 $\bullet$   &0.2302 $\bullet$   &0.1394 $\bullet$\\
\bottomrule[1pt]

\multirow{5}*{USPS}
~&0.1  & 0.3320     &\textbf{0.3364}    &0.2202$\bullet$    &0.1958 $\bullet$   &0.2177 $\bullet$   &0.2467 $\bullet$   &0.1962 $\bullet$   &0.2134     $\bullet$&0.1800 $\bullet$\\
~&0.2& 0.3219   &\textbf{0.3364}    &0.2683 $\bullet$   &0.2460 $\bullet$&  0.2738 $\bullet$    &0.2819 $\bullet$   &0.2459 $\bullet$   &0.2308     $\bullet$&0.1931 $\bullet$ \\
~&0.3  & \textbf{0.3465}    &0.3364     &0.2738 $\bullet$   &0.2656 $\bullet$   &0.2879 $\bullet$   &0.2854 $\bullet$   &0.2656 $\bullet$   &0.2357 $\bullet$   &0.1747 $\bullet$ \\
~&0.4  & \textbf{0.3419}    &0.3364     &0.2713 $\bullet$   &0.2823 $\bullet$   &0.2877 $\bullet$   &0.3051 $\bullet$   &0.2841 $\bullet$   &0.2437     $\bullet$&0.1852$\bullet$ \\
~&0.5  & \textbf{0.3516}    &0.3364 $\bullet$   &0.2775 $\bullet$   &0.2908$\bullet$ &  0.2853$\bullet$&    0.3414  &0.2942 $\bullet$   &0.2039 $\bullet$   &0.1826 $\bullet$\\
\bottomrule[1pt]
\end{tabular}
}
\end{table*}

\begin{figure*}[!htbp]
\centering
\includegraphics[width=\textwidth]{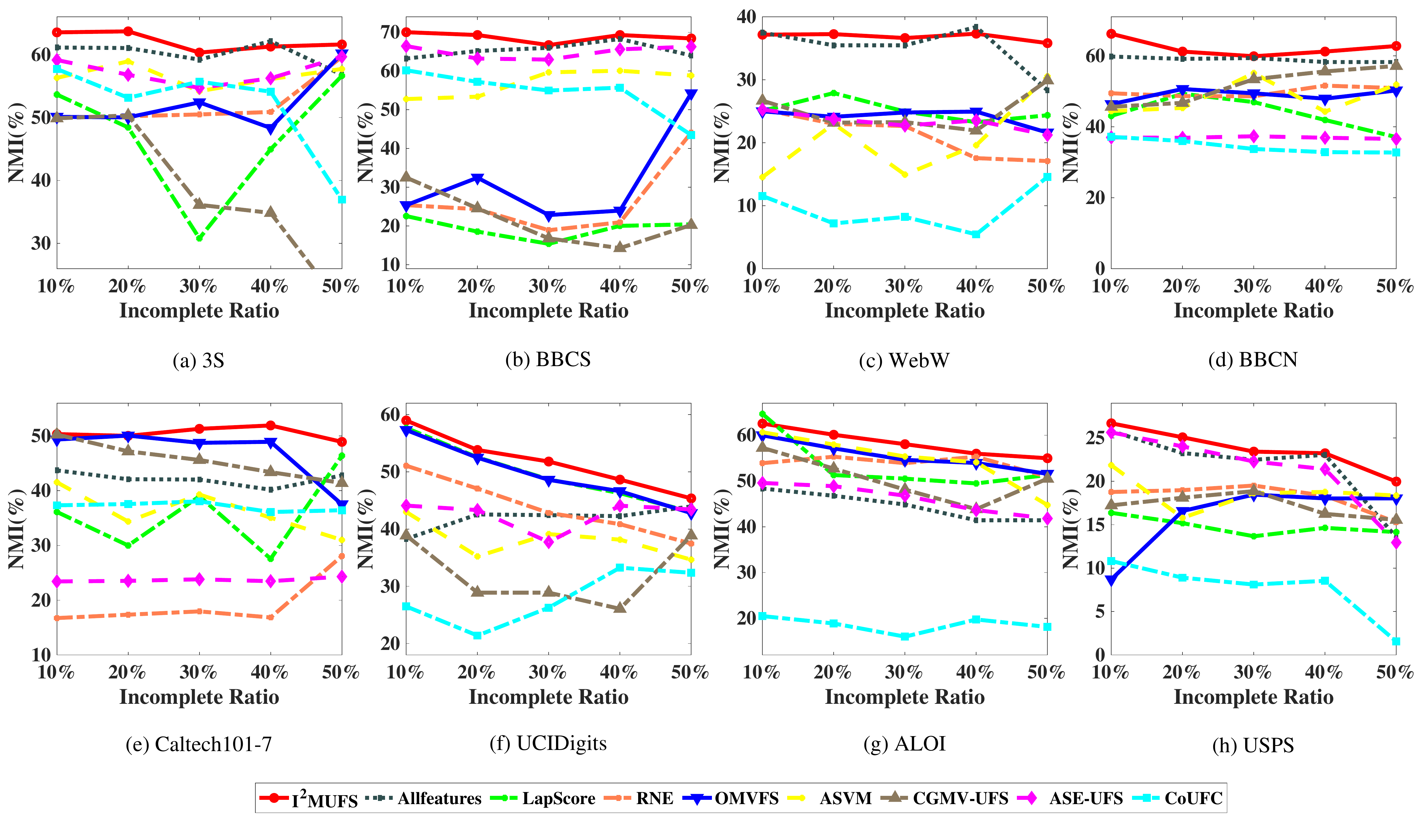}
\caption{NMI of different methods on eight datasets versus different incomplete ratios.}\label{NMI}
\end{figure*}

\begin{figure*}[!htbp]
\centering
\includegraphics[width=\textwidth]{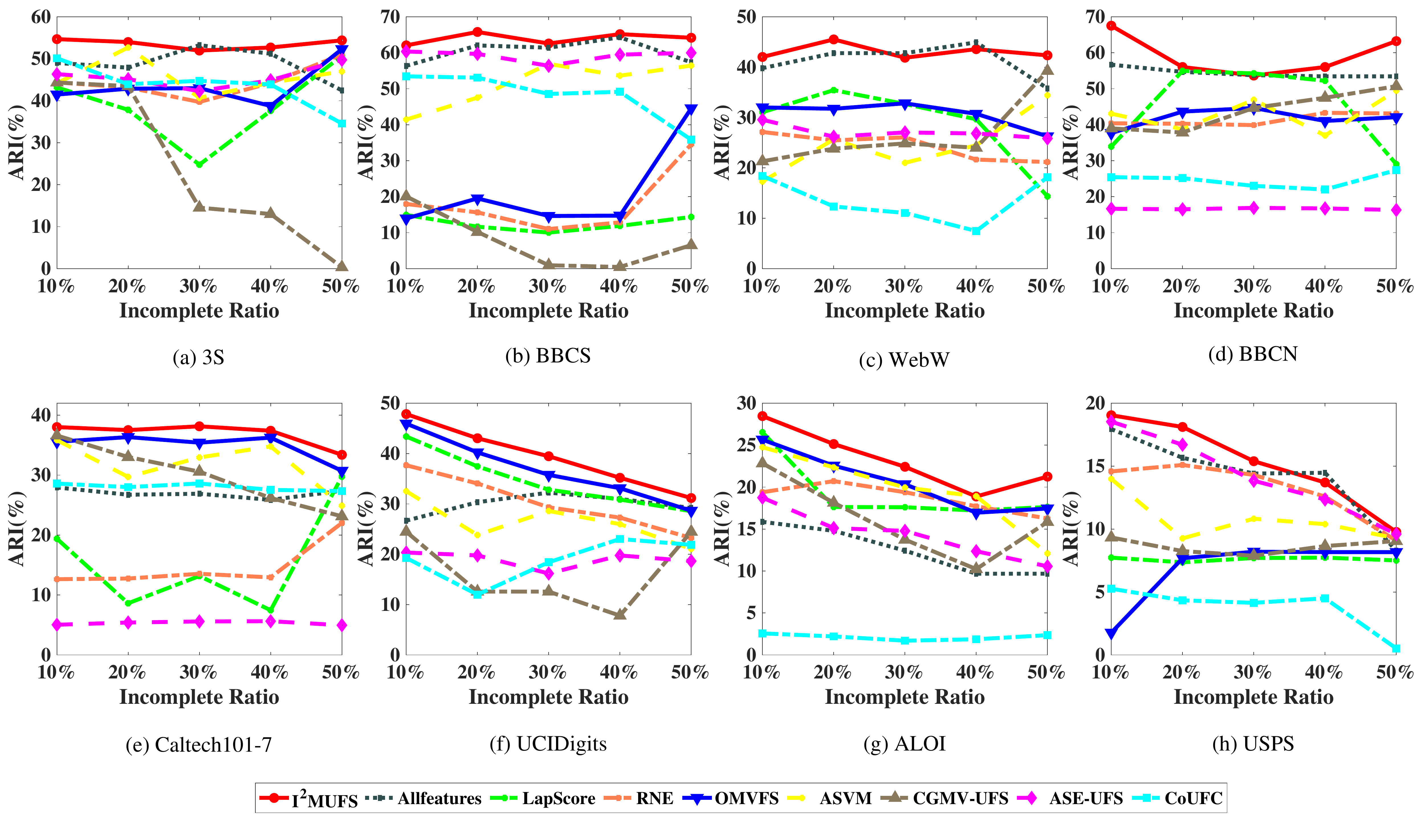}
\caption{ARI of different methods on eight datasets  versus different incomplete ratios.}\label{ARI}
\end{figure*}

\begin{figure*}[!htbp]
\centering
\includegraphics[width=\textwidth]{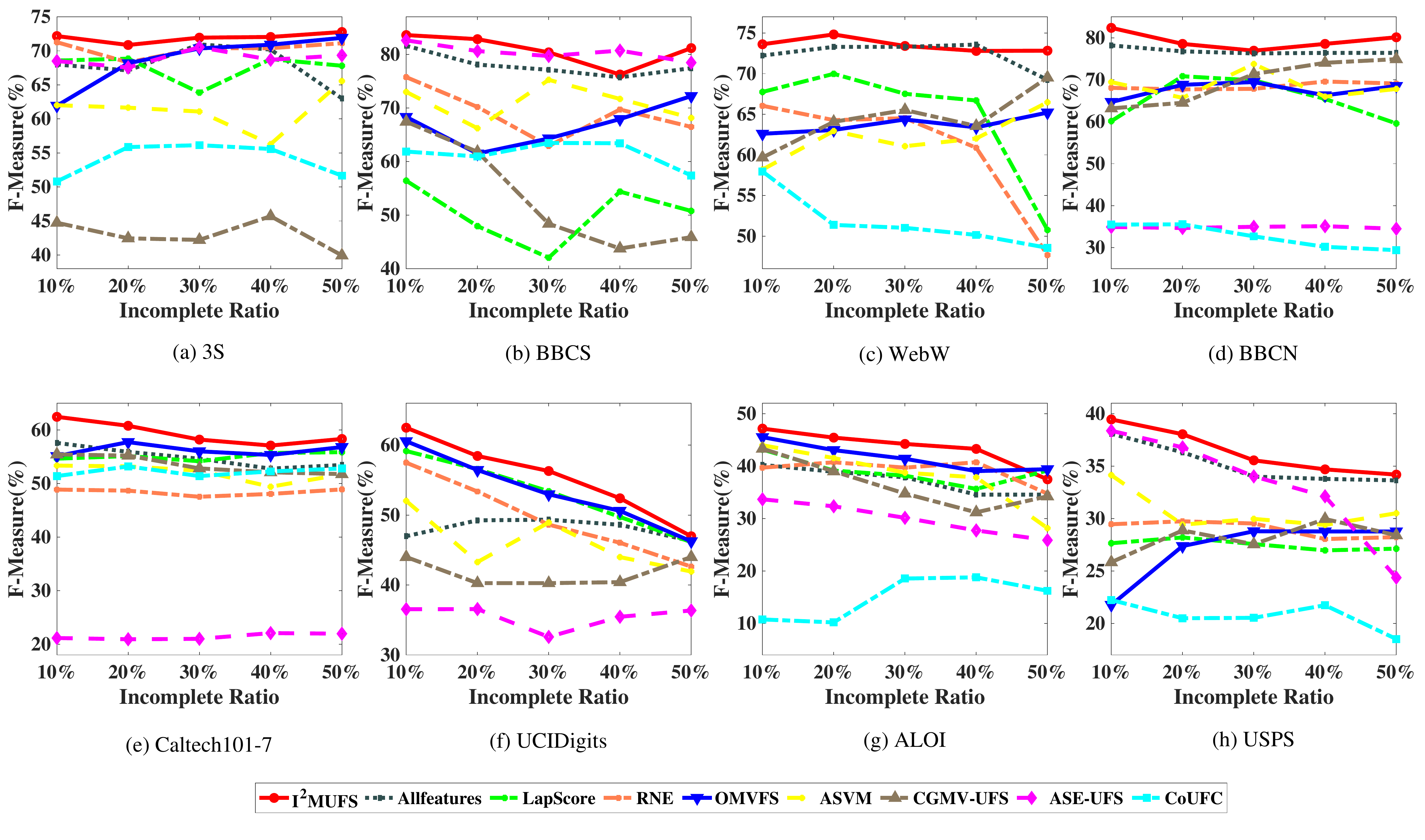}
\caption{F-Measure of different methods on eight datasets  versus different incomplete ratios.}\label{FMeasure}
\end{figure*}

\subsection{Clustering Performance Comparison}
In this subsection, we evaluate the performance of the proposed method in the clustering task. Tables~\ref{Table2},~\ref{Table3} and \ref{Table4} respectively show the clustering results NMI, ARI and F-Measure of compared methods with different selected feature ratios (RFS) while the incomplete instance ratio is set to 50\%. The best experimental result in Tables~\ref{Table2},~\ref{Table3} and \ref{Table4} is highlighted with bold-face type. Besides, with the significance level of 0.05, we employ the paired $t$-test to test the significant difference between the proposed method I$^2$MUFS and other compared methods. The value with $\bullet$ in these three tables indicates that I$^2$MUFS is statistically superior to other method.  From this table, we can see that the proposed method I$^2$MUFS outperforms all other methods on all datasets with different feature selection ratios in most of the time. As to 3S, WebW and Caltech 101-7 datasets, our method I$^2$MUFS performs the best in comparison with other methods and achieves almost 10\% average improvement in terms of NMI, ARI and F-Measure. As to ALOI dataset, I$^2$MUFS gets the best performance according to NMI and ARI. As to  BBCS, BBCN, UCIDigits and USPS datasets, our method  is still better than other methods except for Allfeatures on the feature selection ratios 10\% and 20\%.  Notice that Allfeature is not a feature selection method. Besides, on these four datasets, I$^2$MUFS gets almost the same performance comparing with Allfeatures even with very small ratios of features (10\% or 20\%). Figs.~\ref{NMI},~\ref{ARI} and \ref{FMeasure} respectively show the  NMI, ARI and F-Measure values of all methods on eight datasets with different incomplete instance ratios while fixing 40\% of selected features. It is easy to observe that I$^2$MUFS is still better than other methods in most of the time. In conclusion, these experimental results show that our method I$^2$MUFS outperforms other compared methods in most cases. The reason of I$^2$MUFS with the promising performance is that our method simultaneously considers the consensus and complementary information across different views, where the learned consensus clustering indicator matrix and the fused latent feature matrices with adaptive view weights are beneficial to improve the performance of feature selection.

\begin{figure*}[!htbp]
\centering
\includegraphics[scale=0.32]{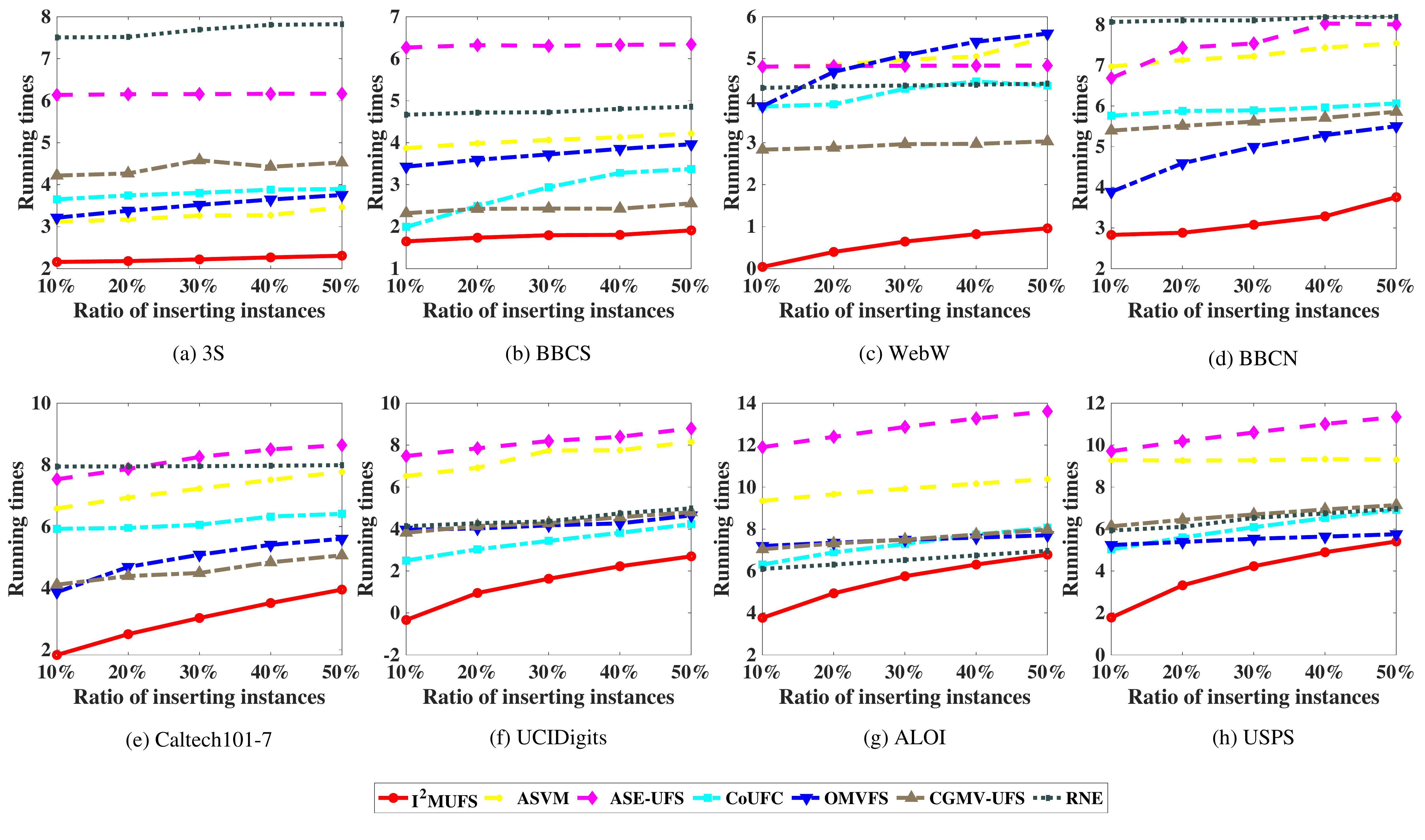}
\caption{Comparison of running time between I$^2$MUFS and other methods versus different inserting ratios.}\label{RunningTime}
\end{figure*}

\subsection{Computational Efficiency Comparison}
In this subsection, we compare the computational cost between the proposed method I$^2$MUFS and other baseline methods while streaming data arrive incrementally. For each dataset in Table~\ref{Table1}, we randomly select 50\% instances from the original dataset  as the initial data and insert different ratios of instances derived from the remaining 50\% instances into the initial data. Let the ratio of inserted instances vary from 10\% to 50\% with a step of 10\% and the new added instances are divided into 5 chunks of equal size. Then we compute the running times of these methods addressing different ratios of new added data in the form of streams. All experiment are repeated 30 times under the same conditions and the average running times are recorded.

Fig.~\ref{RunningTime} shows the detailed change trend of running time w.r.t. different multi-view feature selection methods on eight datasets while inserting different ratios of instances. For each sub-figure in Fig.~\ref{RunningTime}, the $x$-coordinate pertains to the inserting ratio and the $y$-coordinate pertains to the logarithm value of the running times w.r.t. different compared methods. From Fig.~\ref{RunningTime}, it is easy to see that the running times of various methods rise with the increase of the inserting ratios. However, the running time of the proposed method I$^2$MUFS is always faster than that of other compared methods. Furthermore, we use the incremental speedup ratio~\cite{speedup} to show the computational performance of compared methods. The incremental speedup ratio is defined as $IncS = \frac{T_C}{T_I}$, where $T_C$ is the running time of the compared method and $T_I$ is the running time of I$^2$MUFS. The larger value of $IncS$ is, the higher the efficiency of the proposed method. Table~\ref{SpeedupRatio} shows the results of incremental speedup ratio. As shown in Table~\ref{SpeedupRatio}, our method yields 1.4141-73.7824$\times$ speedup over the incremental multi-view feature selection method OMVFS. Comparing with other non-incremental feature selection methods, our method achieves high speedup ratio ($\geq 10\times$) on most of datasets. Hence, the proposed method I$^2$MUFS can effectively improve the computational efficiency in comparison with other feature selection methods while streaming data arrive incrementally.

\begin{table*}[!htbp]
\tabcolsep 0pt
\caption{Incremental speedup ratio with inserting different ratios of instances.}\small \label{SpeedupRatio}
\vspace*{-5pt}
    \begin{flushleft}
    \def\temptablewidth{\textwidth}
        {\rule{\temptablewidth}{1pt}}
        \begin{spacing}{1}
        \begin{tabular*}{\temptablewidth}{@{\extracolsep{\fill}}cccccccc}
      \textbf{Datasets} & \textbf{Adding ratio} & \textbf{RNE} & \textbf{OMVFS} & \textbf{ASVW }  & \textbf{CGMV-UFS } & \textbf{ASE-UFS} & \textbf{CoUFC}\\
       \hline
      \multirow{5}{*}{3S} & 10\% & 210.3975 & 2.8700 & 2.6142 & 7.8281  & 53.4830 & 4.4462 \\
&20\% & 208.3905 & 3.3165 & 2.6905 & 8.0659  & 53.2982 & 4.7845 \\
&30\% & 238.7122 & 3.6705 & 2.8435 & 10.6195 & 51.4127 & 4.8902 \\
&40\% & 255.2657 & 3.9707 & 2.7363 & 8.6578  & 49.3051 & 5.0244 \\
&50\% & 248.6564 & 4.2489 & 3.1822 & 9.2201  & 47.4903 & 4.9048\\
\hline
     \multirow{5}{*}{BBCS} &10\% & 492.1154 & 5.9314 & 9.2223  & 1.9575 & 101.4981 & 1.4108 \\
&20\% & 454.5074 & 6.4005 & 9.4846  & 1.9878 & 98.3053  & 2.1329 \\
&30\% & 429.6427 & 6.8586 & 9.6815  & 1.8854 & 91.1324  & 3.1381 \\
&40\% & 434.8895 & 7.7542 & 10.2561 & 1.8637 & 92.4565  & 4.3842 \\
&50\% & 393.6031 & 7.7872 & 10.0834 & 1.9018 & 84.1765  & 4.3180\\
\hline
 \multirow{5}{*}{WebW} &10\% & 711.2584 & 4.7165 & 116.9978 & 16.3494 & 118.2070 & 45.6664 \\
&20\% & 514.5755 & 3.8024 & 86.6611  & 11.9652 & 83.7772  & 33.6952 \\
&30\% & 412.9379 & 3.3407 & 131.7110 & 10.1994 & 65.9671  & 38.2360 \\
&40\% & 353.4262 & 3.2154 & 63.7243  & 8.5840  & 55.4541  & 38.0132 \\
&50\% & 313.2453 & 3.1396 & 60.2262  & 7.9479  & 48.3865  & 29.9713\\
\hline
 \multirow{5}{*}{BBCN} &  10\% & 188.0958 & 2.8830 & 62.9571 & 12.9986 & 47.3027  & 18.7380 \\
&20\% & 185.1365 & 5.5209 & 69.9184 & 13.8352 & 94.5923  & 19.9859 \\
&30\% & 152.2499 & 6.7918 & 63.1076 & 12.6339 & 86.4621  & 16.6732 \\
&40\% & 133.9070 & 7.3666 & 63.3662 & 11.2666 & 114.5361 & 14.5823 \\
&50\% & 84.3822  & 5.7230 & 44.0993 & 8.1621  & 70.0394  & 10.0356\\
\hline
 \multirow{5}{*}{Caltech101-7} & 10\% & 450.9037 & 7.6687 & 115.7769 & 9.7709 & 298.5746 & 59.8759 \\
&20\% & 230.5780 & 8.8458 & 83.9281  & 6.5866 & 212.3521 & 31.4037 \\
&30\% & 137.5220 & 7.7793 & 66.5833  & 4.3006 & 185.4691 & 20.5628 \\
&40\% & 85.7938  & 6.6005 & 54.2774  & 3.7436 & 145.9295 & 16.5481 \\
&50\% & 56.6196  & 5.1990 & 45.4668  & 3.0371 & 108.5071 & 11.6938\\
     \hline
 \multirow{5}{*}{UCIDigits} & 10\% & 87.2360 & 73.7824 & 956.4561 & 63.9149 & 2473.8623 & 17.0240 \\
&20\% & 27.7769 & 21.6282 & 390.9540 & 24.2648 & 986.0699  & 8.0076  \\
&30\% & 15.4289 & 12.7225 & 452.9216 & 14.1115 & 708.7078  & 6.0862  \\
&40\% & 12.4575 & 7.7186  & 252.8110 & 10.3597 & 479.4795  & 4.9307  \\
&50\% & 9.8592  & 7.0273  & 233.1481 & 8.1774  & 444.7255  & 4.6598 \\
      \hline
 \multirow{5}{*}{ALOI} & 10\% &10.2959   & 30.6821 & 264.9053 & 26.3096 & 3409.6980  & 12.7351 \\
&20\% & 3.9254  & 11.1198 & 112.4695 & 10.7385 & 1734.0379  & 7.0359  \\
&30\% &2.1700   & 5.6827  & 64.7829  & 5.7175  & 1228.0914  & 4.6587  \\
&40\% & 1.5457  & 3.6040  & 47.3404  & 4.1753  & 1057.2306  & 4.1869  \\
&50\% & 1.1943  & 2.5154  & 36.6580  & 3.2767  & 921.3091  & 3.5870 \\
      \hline
 \multirow{5}{*}{USPS} & 10\% & 45.7086 & 31.7256 & 1806.6021 & 77.3680 & 2774.2346  & 25.7583 \\
&20\% & 11.6403 & 7.9029  & 384.5124  & 22.6896 & 962.8740  & 9.7976  \\
&30\% & 5.5454  & 3.6935  & 155.4301  & 11.7434 &585.1557  & 6.4159  \\
&40\% & 3.7571  & 2.1052  & 84.8838   & 7.6687  & 451.6141  & 5.0866  \\
&50\% & 2.9924  & 1.4141  & 49.5503   & 5.6639  & 379.1854  & 4.5374 \\
	\hline
        \end{tabular*}
        {\rule{\temptablewidth}{1pt}}
        \end{spacing}
        \end{flushleft}
\end{table*}
\subsection{Convergence Behavior  and Parameter Sensitivity}
In this subsection, we study the convergence behavior and parameter sensitivity of the proposed method I$^2$MUFS. We have proven the convergence of the proposed method  I$^2$MUFS in Section 5.1. Here, we discuss its convergence speed experimentally.  Fig.~\ref{Convergence2} shows the objective function values of I$^2$MUFS with different number of  iterations on  eight datasets.  As can be seen from this figure, the convergence curve decreases very fast within about 50 iterations, which indicates the proposed method with fast convergence speed empirically.

There are four hyper-parameters $\lambda$, $\beta^{(v)}$, $\theta^{(v)}$ and $\eta^{(v)}$ in our method. To investigate the parameter sensitivity, we perform the proposed algorithm I$^2$MUFS on the Caltech101-7 dataset  while varying a parameter and fixing other rest parameters. On other datasets, there exist the similar results. Fig.~\ref{ParameterSensitivity} shows the NMI, ARI and F-Measure values of I$^2$MUFS with regards to different parameters and the feature selection ratio. From this figure, it is easy to see that NMI, ARI and F-Measure have a small fluctuation range when $\beta^{(v)}$, $\theta^{(v)}$ and $\eta^{(v)}$ change. Namely, I$^2$MUFS is insensitive to these three parameters. Besides, when $\lambda$ changes, the fluctuations of ARI and F-Measure are severer than that of NMI, which means that the ARI and F-Measure values are relatively sensitive to $\lambda$. Hence, we need tune the parameter $\lambda$ to achieve better ARI and F-Measure by using the grid search.

\begin{figure*}[!htbp]
\begin{center}
\includegraphics[scale=0.32]{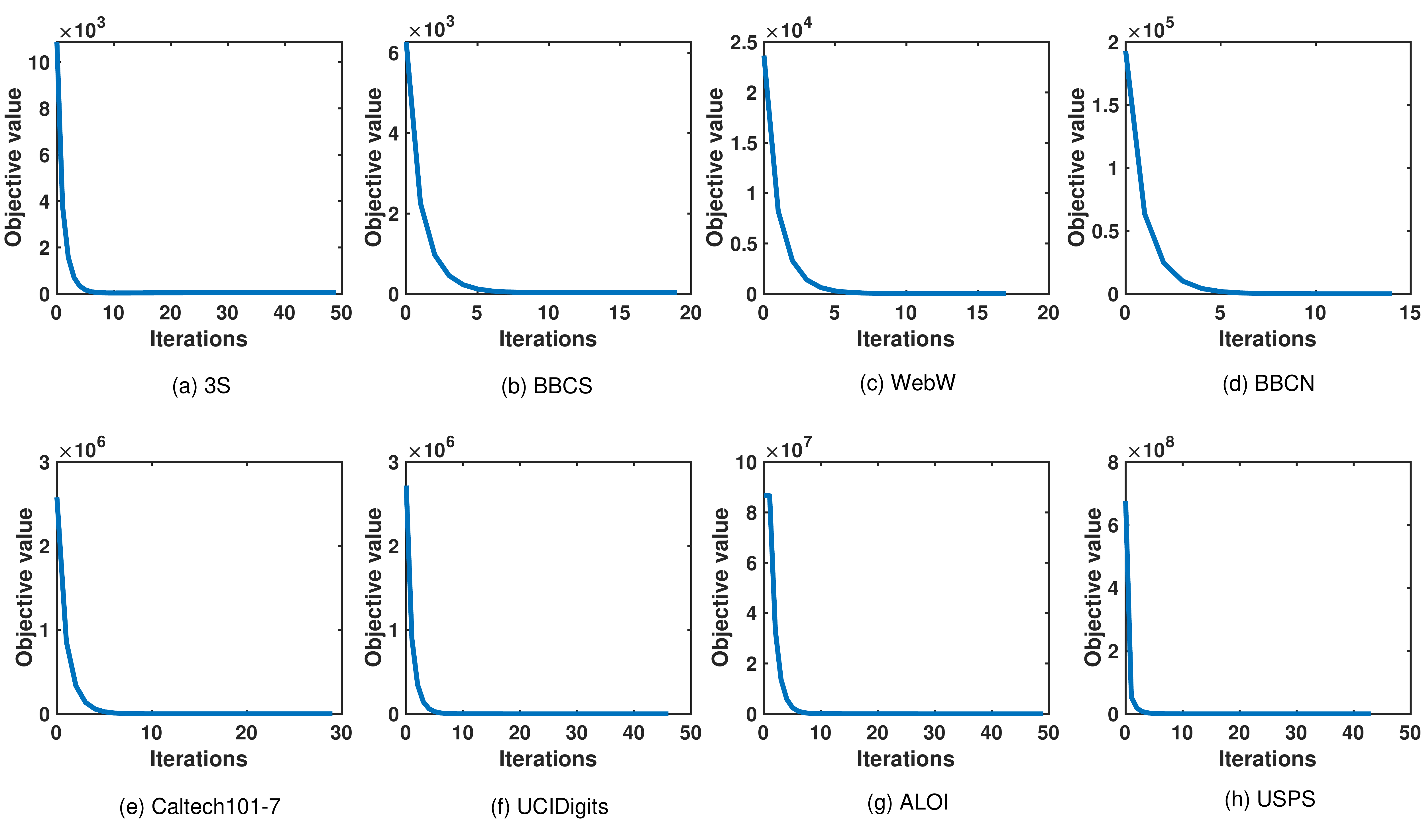}
\end{center}
\caption{Convergence curves of I$^2$MUFS on eight datasets.}\small \label{Convergence2}
\end{figure*}

\begin{figure*}[!htbp]
\begin{center}
\includegraphics[width=\textwidth]{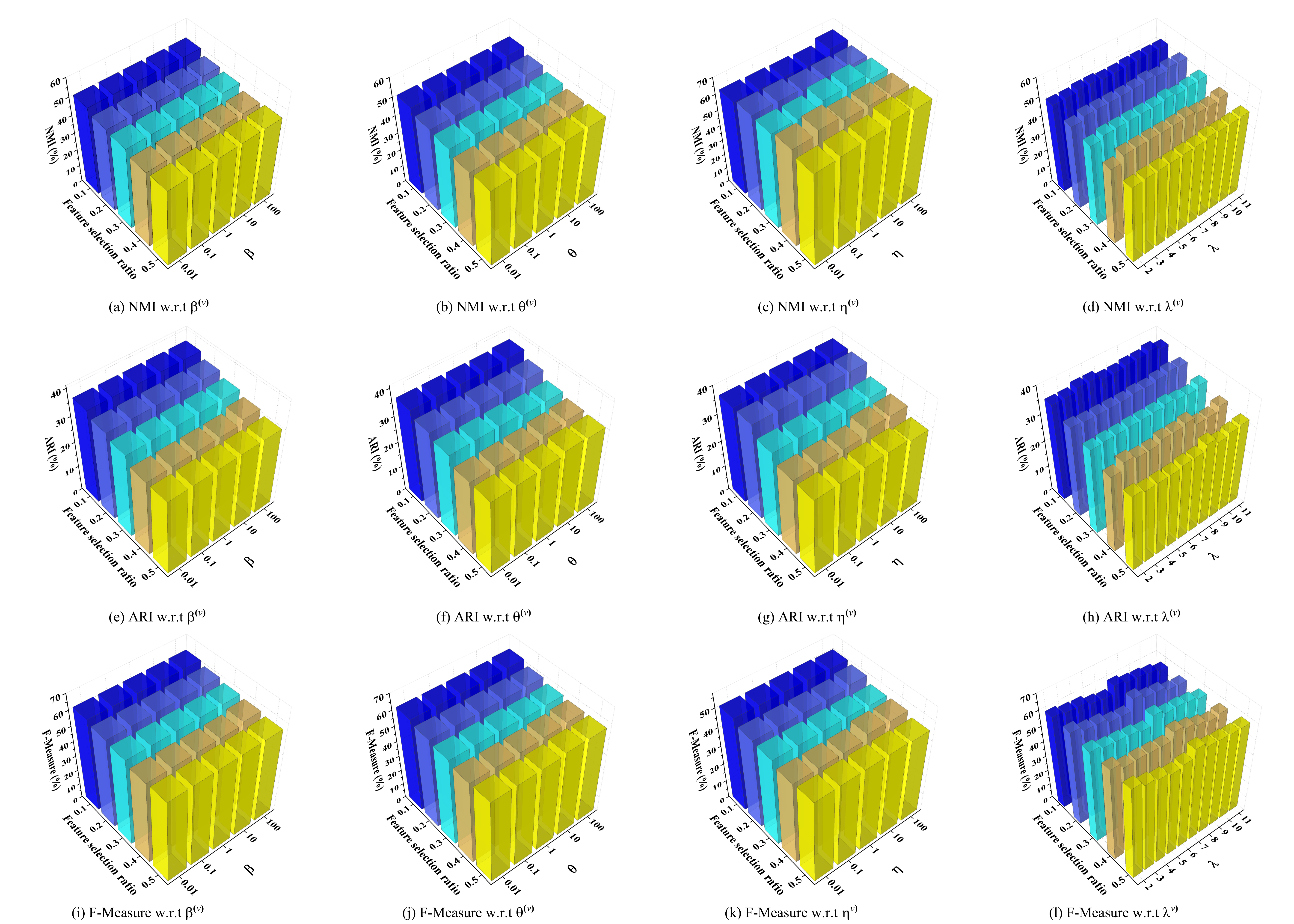}
\end{center}
\caption{NMI, ARI and F-Measure of I$^2$MUFS on Caltech101-7 dataset with different parameters.}\small \label{ParameterSensitivity}
\end{figure*}

\subsection{Ablation Study}
In this section, we conduct an ablation study to further verify the effectiveness of the extended WNMF module in the proposed method. To do so, we use the classical WNMF module~\cite{WNMF} to replace the first term (the extended WNMF module) in  Eq.~(\ref{Eq8}). Then, a variant of I$^2$MUFS called C-I$^2$MUFS is obtained as follows:

\begin{align}\label{EQ44}
&\min_{\textbf{U}_{t}^{(v)}, \textbf{U}_{t}^{*}, \textbf{V}^{(v)}}\sum_{v=1}^{n_v}\sum_{t=1}^{T} (   \| (\textbf{X}_{t}^{(v)}-\textbf{V}^{(v)}{\textbf{U}_{t}^{(v)}}^{T})\textbf{A}_{t}^{(v)}  \|_{F}^{2}  +\beta^{(v)} \| \textbf{W}_{t}^{(v)}(\textbf{U}_{t}^{(v)}-\textbf{U}_{t}^{*})  \|_{F}^{2} +
\theta^{(v)}Tr({\textbf{U}_{t}^{(v)}}^{T}\textbf{L}_{t}^{(v)}\textbf{U}_{t}^{(v)}) )  +\eta^{(v)} \| \textbf{V}^{(v)}  \|_{2,1} \notag   \\
&s.t. \textbf{U}_{t}^{(v)}\ge 0, \textbf{V}^{(v)}\ge 0, \textbf{U}_{t}^{*} \ge 0, {\textbf{U}_t^{(v)}}^{T}\textbf{U}_t^{(v)}=\textbf{I} ,\alpha^{(v)} \ge 0, \sum_{v=1}^{n_v}\alpha^{(v)}=1,
\end{align}
where the $j$-th diagonal entry of the diagonal matrix $\textbf{A}_{t}^{(v)}$ equals to 1, if the $j$-th instance is in the $v$-th view, 0 otherwise.

Table~\ref{Table6} shows the ablation experiment results. From this table, we can see that the proposed method I$^2$MUFS is superior to its variant C-I$^2$MUFS  on all data-sets in terms of NMI, ARI and F-Measure. The notable results demonstrate the extended WNMF module considering both the quality of available information until now and the fusion of views with adaptive weights can improve the performance of incomplete multi-view unsupervised feature selection.

\begin{table*}[!htbp]
\centering
\caption{NMI, ARI and F-Measure values over C-I$^2$MUFS and I$^2$MUFS with different feature ratios}\label{Table6}
\vspace{-0.2cm}
\resizebox{\textwidth}{!}{
\renewcommand\tabcolsep{5 pt}
\begin{tabular}{ccccccccccccccccccc}
\toprule[1pt]
\multirow{2}*{Metrics} & \multirow{2}*{RFS} & \multicolumn{2}{c}{3S}&\multicolumn{2}{c}{BBCS} &\multicolumn{2}{c}{WebW}&\multicolumn{2}{c}{BBCN} &\multicolumn{2}{c}{Caltech101-7}&\multicolumn{2}{c}{UCIDigits} &\multicolumn{2}{c}{ALOI} &\multicolumn{2}{c}{USPS}& \\
\cmidrule(r){3-4} \cmidrule(r){5-6} \cmidrule(r){7-8} \cmidrule(r){9-10} \cmidrule(r){11-12} \cmidrule(r){13-14} \cmidrule(r){15-16} \cmidrule(r){17-18}
~&~&I$^2$MUFS&C-I$^2$MUFS&I$^2$MUFS&C-I$^2$MUFS&I$^2$MUFS&C-I$^2$MUFS&I$^2$MUFS&C-I$^2$MUFS&I$^2$MUFS&C-I$^2$MUFS&I$^2$MUFS&C-I$^2$MUFS&I$^2$MUFS&C-I$^2$MUFS&I$^2$MUFS&C-I$^2$MUFS&\\
\midrule[0.5pt]
\multirow{5}*{NMI}
~& 0.1 &\textbf{0.5974} & 0.5655$\bullet$ &\textbf{0.6314} &0.5825 $\bullet$&\textbf{0.3008} & 0.2165$\bullet$&\textbf{0.5449} &0.5196 $\bullet$&\textbf{0.5188} & 0.4102$\bullet$&\textbf{0.3955} &0.3693 $\bullet$&\textbf{0.5330} & 0.4899$\bullet$&\textbf{0.1555} &0.1508 $\bullet$  \\
~& 0.2 &\textbf{0.6127} & 0.5713$\bullet$&\textbf{0.6667} &0.6418$\bullet$ &\textbf{0.3213} & 0.2565$\bullet$&\textbf{0.5794} & 0.5361$\bullet$&\textbf{0.5183} & 0.4384$\bullet$&\textbf{0.4387} &0.4103 $\bullet$&\textbf{0.5303} & 0.4980$\bullet$&\textbf{0.1839} &0.1617 $\bullet$\\
~& 0.3 &\textbf{0.6066} & 0.5737$\bullet$ &\textbf{0.6840} &0.6683 $\bullet$&\textbf{0.3502} & 0.2620$\bullet$&\textbf{0.6046} &0.5746$\bullet$ &\textbf{0.5004} & 0.4479$\bullet$&\textbf{0.4577} &0.4323$\bullet$ &\textbf{0.5468} & 0.5049$\bullet$&\textbf{0.1933} &0.1708 $\bullet$  \\
~& 0.4 &\textbf{0.6161} & 0.5843$\bullet$ &\textbf{0.6842} &0.6652 $\bullet$&\textbf{0.3582} & 0.3201$\bullet$&\textbf{0.6277} & 0.5716$\bullet$&\textbf{0.4839} & 0.4596$\bullet$&\textbf{0.4536} & 0.4394$\bullet$&\textbf{0.5497} & 0.5110$\bullet$&\textbf{0.1997} & 0.1702 $\bullet$ \\
~& 0.5 &\textbf{0.6114} & 0.5874$\bullet$ &\textbf{0.6749} &0.6528$\bullet$ &\textbf{0.3468} & 0.3070$\bullet$&\textbf{0.6313} & 0.6101$\bullet$&\textbf{0.4814} & 0.4413$\bullet$&\textbf{0.4550} & 0.4308$\bullet$&\textbf{0.5467} & 0.5063$\bullet$&\textbf{0.2190} &0.1717 $\bullet$  \\
\bottomrule[1pt]

\multirow{5}*{ARI}
~& 0.1 &\textbf{0.5135} & 0.4973$\bullet$ &\textbf{0.6034} & 0.5357$\bullet$&\textbf{0.3699} & 0.2963$\bullet$&\textbf{0.5171} &0.4692 $\bullet$&\textbf{0.3805} & 0.2572$\bullet$&\textbf{0.2709} &0.2535 $\bullet$&\textbf{0.1629} & 0.1334$\bullet$&\textbf{0.0842} &0.0459$\bullet$   \\
~& 0.2 &\textbf{0.5372} & 0.5130$\bullet$&\textbf{0.6417} &0.6086 $\bullet$&\textbf{0.3896} & 0.3317$\bullet$&\textbf{0.5598} & 0.5069$\bullet$&\textbf{0.3590} & 0.2749$\bullet$&\textbf{0.2980} &0.2681 $\bullet$&\textbf{0.1972} & 0.1392$\bullet$&\textbf{0.0896} &0.0511$\bullet$  \\
~& 0.3 &\textbf{0.5539} & 0.5233$\bullet$ &\textbf{0.6574} &0.6218$\bullet$ &\textbf{0.4237} & 0.3382$\bullet$&\textbf{0.5910} & 0.5520$\bullet$&\textbf{0.3422} & 0.2830$\bullet$&\textbf{0.3176} &0.2810$\bullet$ &\textbf{0.2104} & 0.1467$\bullet$&\textbf{0.0935} &0.0709 $\bullet$  \\
~& 0.4 &\textbf{0.5439} & 0.5252$\bullet$ &\textbf{0.6418} &0.6017 $\bullet$&\textbf{0.4235} & 0.3877$\bullet$&\textbf{0.6324} & 0.5051$\bullet$&\textbf{0.3339} & 0.2917$\bullet$&\textbf{0.3116} & 0.2891$\bullet$&\textbf{0.2125} & 0.1487$\bullet$&\textbf{0.0977} & 0.0791$\bullet$ \\
~& 0.5 &\textbf{0.5692} & 0.5028$\bullet$ &\textbf{0.6459} &0.6185 $\bullet$&\textbf{0.4256} & 0.3879$\bullet$&\textbf{0.6566} &0.6309 $\bullet$&\textbf{0.3098} & 0.2764$\bullet$&\textbf{0.3156} &0.2902$\bullet$ &\textbf{0.2092} & 0.1446$\bullet$&\textbf{0.1185} &0.0824 $\bullet$  \\
\bottomrule[1pt]

\multirow{5}*{F-Measure}
~& 0.1 &\textbf{0.7037} & 0.6748$\bullet$ &\textbf{0.7991} & 0.7519$\bullet$&\textbf{0.6970} & 0.6545$\bullet$&\textbf{0.7412} &0.7153$\bullet$ &\textbf{0.6282} & 0.5160$\bullet$&\textbf{0.4796} &0.4066 $\bullet$&\textbf{0.4068} & 0.3630$\bullet$&\textbf{0.3320} &0.2707 $\bullet$ \\
~& 0.2 &\textbf{0.7177} & 0.6908$\bullet$ &\textbf{0.8068} &0.7612$\bullet$ &\textbf{0.7027} & 0.6690$\bullet$&\textbf{0.7740} &0.7497 $\bullet$&\textbf{0.6070} & 0.5337$\bullet$&\textbf{0.4845} &0.4378$\bullet$ &\textbf{0.3711} & 0.3502$\bullet$&\textbf{0.3219} &0.2823 $\bullet$  \\
~& 0.3 &\textbf{0.7215} & 0.6987$\bullet$ &\textbf{0.8219} &0.7703$\bullet$ &\textbf{0.7206} & 0.6884$\bullet$&\textbf{0.7872} &0.7503$\bullet$ &\textbf{0.5897} & 0.5476$\bullet$&\textbf{0.4992} &0.4523 $\bullet$&\textbf{0.3794} & 0.3483$\bullet$&\textbf{0.3465} &0.2947 $\bullet$  \\
~& 0.4 &\textbf{0.7277} & 0.7043$\bullet$ &\textbf{0.8176} &0.7873 $\bullet$&\textbf{0.7283} & 0.7063$\bullet$&\textbf{0.8012} & 0.7495$\bullet$&\textbf{0.5832} & 0.5524$\bullet$&\textbf{0.4697} & 0.4370$\bullet$&\textbf{0.3748} & 0.3435$\bullet$&\textbf{0.3419} & 0.2853 $\bullet$ \\
~& 0.5 &\textbf{0.7389} & 0.6944$\bullet$ &\textbf{0.8169} &0.7909$\bullet$ &\textbf{0.7276} & 0.7064$\bullet$&\textbf{0.8096} & 0.7306$\bullet$&\textbf{0.5783} & 0.5491$\bullet$&\textbf{0.4681} & 0.4322$\bullet$&\textbf{0.3730} & 0.3521$\bullet$&\textbf{0.3516} &0.2823 $\bullet$  \\
\bottomrule[1pt]
\end{tabular}
}
\end{table*}

\section{Conclusion}
In this paper, we proposed a novel multi-view unsupervised feature selection method I$^2$MUFS to deal with the incomplete multi-view streaming data. I$^2$MUFS learns the common clustering indicator matrix and fuses the latent feature matrices with adaptive view weights by simultaneously considering the consistency and complementarity between views. Besides, by the utilization of incremental learning mechanisms, we developed an alternative iterative optimization algorithm to improve the efficiency of feature selection. Experimental results on several real-world datasets show that I$^2$MUFS outperforms the SOTA unsupervised feature selection methods in terms of the clustering metrics and the running time.  Although the proposed I$^2$MUFS obtains the best performance on ALOI and USPS datasets in terms of ARI, it gets relatively low performance compared to other datasets. Hence,  future work will try to extend our method to improve the performance on these two datasets. In addition, even though it is hard to obtain a lot of labeled instances, we may collect a little labeled data in real applications. In the future work, we will also investigate the semi-supervised feature selection method for the dynamic incomplete multi-view data.

\section*{Acknowledgments}
This work was supported by the  Youth Fund Project of Humanities and Social Science Research of Ministry of Education (No. 21YJCZH045), National Science Foundation of China (Nos. 61773324, 62076171), the Fundamental Research Funds for the Central Universities (No. JBK2101001), the Natural Science Foundation of  Fujian Province (No. 2020J01800) and the Joint Lab of Data Science and Business Intelligence at Southwestern University of Finance and Economics.

\end{document}